\PassOptionsToPackage{table,dvipsnames}{xcolor}
\PassOptionsToPackage{dvipsnames}{colortbl}
\documentclass{article} 
\usepackage{iclr2027_conference,times}
\usepackage{enumerate}
\usepackage{fontawesome5}

\usepackage{amsmath,amsfonts,bm}

\def\1{\bm{1}}

\def\re{{\textnormal{e}}}

\DeclareMathAlphabet{\mathsfit}{\encodingdefault}{\sfdefault}{m}{sl}
\SetMathAlphabet{\mathsfit}{bold}{\encodingdefault}{\sfdefault}{bx}{n}

\usepackage{booktabs} 
\usepackage{multirow} 
\usepackage{longtable} 
\usepackage{colortbl} 
\usepackage{arydshln} 
\usepackage{graphicx} 
\usepackage{amsmath,amssymb}
\usepackage{amsthm}
\usepackage{enumitem} 
\usepackage{subcaption} 
\usepackage{wrapfig} 
\usepackage{needspace} 
\usepackage{tikz} 
\usetikzlibrary{arrows.meta,positioning,calc,fit,backgrounds}
\usepackage{listings} 
\usepackage{tcolorbox}
\tcbuselibrary{listings,breakable}
\definecolor{arsenic}{rgb}{0.23,0.27,0.29}
\lstdefinestyle{promptstyle}{
  basicstyle=\ttfamily\footnotesize\color{black},
  columns=fullflexible,
  keepspaces=true,
  breaklines=true,
  showstringspaces=false,
  tabsize=2,
}
\newtcblisting{promptbox}{
  listing only,
  unbreakable,
  colback=white,
  colframe=black!15,
  boxrule=0.5pt,
  arc=2pt,
  left=6pt,right=6pt,top=6pt,bottom=6pt,
  listing options={style=promptstyle}
}
\usepackage{hyperref}
\usepackage{url}

\definecolor{amiePurple}{HTML}{6B5B95}
\definecolor{amiePurpleBg}{HTML}{EDE7F6}
\definecolor{amieNeutral}{HTML}{8A8A8A}
\colorlet{amieGain}{amiePurple}
\colorlet{amieLoss}{amiePurple}

\definecolor{amiePositiveCell}{HTML}{719B7D}
\definecolor{amieNegativeCell}{HTML}{BA7D7D}
\newcommand{\amieCell}[2]{%
  \pgfmathsetmacro{\amieTint}{min(40,8+8*abs(#2))}%
  \ifdim #2pt>0pt
    \edef\amieShade{\noexpand\cellcolor{amiePositiveCell!\amieTint!white}}%
  \else\ifdim #2pt<0pt
    \edef\amieShade{\noexpand\cellcolor{amieNegativeCell!\amieTint!white}}%
  \else
    \def\amieShade{\cellcolor{white}}%
  \fi\fi
  \amieShade\ensuremath{#1_{#2}}}


\definecolor{darkcerulean}{rgb}{0.03, 0.27, 0.49}
\definecolor{burgundy}{rgb}{0.7, 0.0, 0.13}
\definecolor{majorelleblue}{rgb}{0.38, 0.31, 0.86}
\definecolor{forestgreen}{rgb}{0.13, 0.55, 0.13}
\definecolor{arylideyellow}{rgb}{0.89, 0.61, 0.06}
\definecolor{arsenic}{rgb}{0.23, 0.27, 0.29}
\hypersetup{
  colorlinks=true,
  citecolor=darkcerulean,
  linkcolor=burgundy,
  urlcolor=majorelleblue
}
\newcommand{\LinkTable}[1]{\hyperref[#1]{Table~\ref*{#1}}}
\newcommand{\LinkFigure}[1]{\hyperref[#1]{Figure~\ref*{#1}}}
\newcommand{\LinkFigurePanel}[2]{\hyperref[#1]{Figure~\ref*{#1}#2}}
\newcommand{\LinkSection}[1]{\hyperref[#1]{Section~\ref*{#1}}}
\newcommand{\LinkAppendix}[1]{\hyperref[#1]{Appendix~\ref*{#1}}}
\newcommand{\LinkEquation}[1]{\hyperref[#1]{Equation~\ref*{#1}}}
\newcommand{\LinkProposition}[1]{\hyperref[#1]{Proposition~\ref*{#1}}}
\newcommand{\LinkCorollary}[1]{\hyperref[#1]{Corollary~\ref*{#1}}}
\newcommand{\LinkAlgorithm}[1]{\hyperref[#1]{Algorithm~\ref*{#1}}}

\newtheorem{proposition}{Proposition}

\title{Adaptive Memory Index Evolution for Agents}

\author{
Yuntong Hu\textsuperscript{1},
Tongli Su\textsuperscript{1},
Yuanfu Sun\textsuperscript{2},
Shuo Ma\textsuperscript{3},
Qiaoyu Tan\textsuperscript{2},
Liang Zhao\textsuperscript{1}
\\
\textsuperscript{1}Emory University, 
\textsuperscript{2}New York University, 
\textsuperscript{3}Meituan,
\\
\texttt{\{yuntong.hu,liang.zhao\}@emory.edu}
}

\iclrfinalcopy

\newcommand{\amieUtilityWins}{15}

\newcommand{\amieRetrievalWins}{10}

\begin{document}

\maketitle

\begin{abstract}
Long-horizon agents continually accumulate experience, yet retaining a memory
does not guarantee access when it is needed. Existing systems can adapt both
memory content and access; we isolate persistent access updates as a separately
controlled adaptation layer. We formulate this target as \emph{Memory Index
Evolution} and instantiate it with \emph{Adaptive Memory Index Evolution}
(AMIE), a lightweight plug-in layer that evolves per-memory access state from
deployment retrievals without modifying stored content or retraining the host
memory system. AMIE distils a semantic teacher's keep/suppress judgments into
a retrieval-conditioned policy and accumulates its selective actions into
bounded state that reweights subsequent host retrievals. Surprisingly, even
a zero-call policy that simply accumulates retrieval frequency can match or
outperform semantic policies in several settings, revealing retrieval traffic
itself as a strong signal for access adaptation, while learned selectivity
improves robustness when frequently retrieved memories remain useful. Across
five retrieval-based hosts and three benchmarks, trained AMIE improves
downstream F1 or accuracy in all 15 comparisons, by 0.70--3.37 percentage
points (up to 25\% relative), and improves retrieval in all 10 comparisons.
Controlled analyses show that persistent competitor suppression can expose previously unseen evidence, while suppressing repeatedly useful memories creates the complementary failure mode. Together, these results establish persistent memory access as a practical
adaptation layer, distinct from the content it serves, in retrieval-based agent memory systems.~~\small
    \href{https://github.com/HuieL/AMIE}{
        \faGithub\enspace\texttt{Code}
    }
\end{abstract}

\begin{figure}[htbp]
\centering
\includegraphics[width=\linewidth]{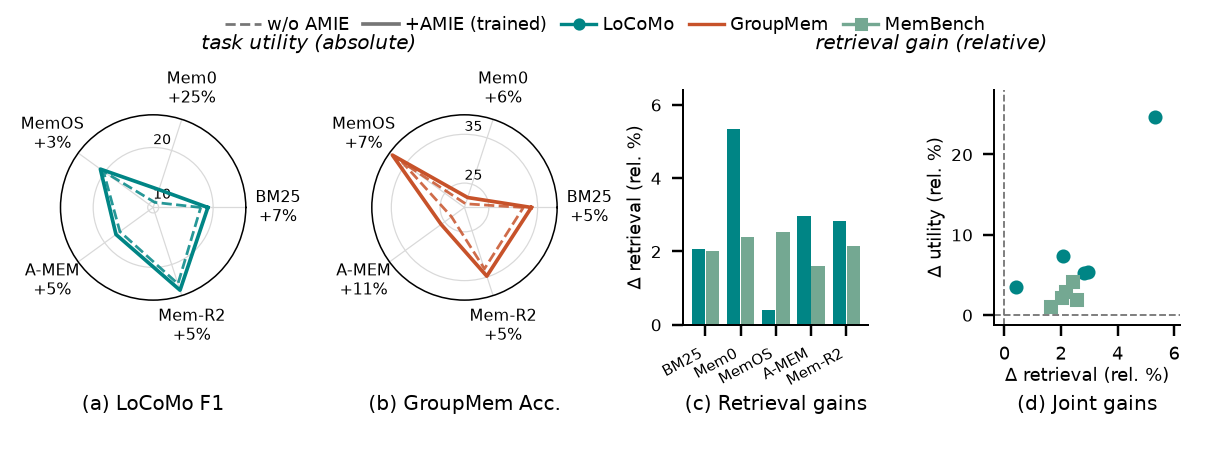}
\caption{\textbf{Persistent index adaptation improves both memory access and downstream utility across heterogeneous memory systems.} 
We evaluate three complementary memory settings: LoCoMo for long-horizon conversation, GroupMem for cross-channel group memory, and MemBench for personal memory.
Trained AMIE improves all downstream host--benchmark comparisons and all comparisons with retrieval measurements. (a,b) Absolute LoCoMo F1 and GroupMem accuracy, with relative gains over each matched host shown at the vertices. (c) Relative retrieval gains on LoCoMo and MemBench. (d) All ten pairs measured at both levels improve in both retrieval and downstream utility, showing that the benefits of index evolution extend beyond retrieval metrics to the tasks memory ultimately serves. 
}
\label{fig:ga-summary}
\end{figure}

\section{Introduction}
\label{sec:intro}

Long-horizon agents accumulate observations, decisions, and successful
procedures that they must recover when a later task needs them
\citep{yao2023react,shinn2023reflexion,packer2023memgpt,maharana2024locomo,wang2023voyager}.
Retaining this experience is only part of the problem: stored information
can remain intact while becoming difficult to retrieve. Consider an agent
that retains the episode that solved a software failure, the kind of
recurring task that realistic agent benchmarks now probe
\citep{jimenez2024swebench,zhou2024webarena,merrill2026terminalbench}. As similar stack
traces and configuration notes accumulate, they may rank above that episode
when the failure returns. The useful experience has not been deleted or
rewritten; access to it has changed.

Memory systems address long-term use through extraction, consolidation,
tiers, graphs, and learned memory operations
\citep{zhong2024memorybank,mem0,xu2025amem,gmemory,yan2025memoryr1,atommem2026}.
Existing systems may adapt both memory content and access, including
recall-dependent strength, access recency, and persistent graph utilities
\citep{zhong2024memorybank,park2023generative,expgraph2026}.
We instead isolate persistent index updates as a separately controlled layer
over the host retrieval interface. This separation asks a specific question:
\begin{tcolorbox}[colback=black!5, colframe=black!18, boxrule=0.4pt, arc=2pt,
  left=6pt, right=6pt, top=5pt, bottom=5pt]
\centering
\emph{Can deployment experience improve memory access without changing stored content?}
\end{tcolorbox}

\emph{Memory Index Evolution} makes persistent access state an explicit
adaptation target, distinct from stored content. Two agents with identical
stores can encounter different query streams and benefit from different
access states. A finite retrieval context makes selection consequential
\citep{liu2023lostmiddle}, but the adaptation target is the state carried
between queries, rather than
a new context budget or a replacement query-time ranker.

\emph{Adaptive Memory Index Evolution} (AMIE) instantiates this layer with
a trained index policy that learns which retrieved memories should accumulate
suppression. We distil a semantic teacher's judgments into a compact action
model, then carry its decisions forward through bounded per-memory state.
The current answer context is sealed before an update is committed; only
later host retrievals use the new state. We also find that replacing learned
selection with unconditional counting can approach frozen-policy performance
on LoCoMo BM25 and MemoryOS, showing how much access improvement retrieval
traffic alone can supply in those settings. Selective actions address the complementary risk of
penalising memories that remain useful.

Across five retrieval-based hosts and three benchmarks, trained AMIE improves
the reported utility scores; the ten pairs with retrieval measurements also
show retrieval gains
(\LinkFigure{fig:ga-summary}).
Matched host comparisons assess utility; controlled replays isolate the
benefits and risks of suppressing recurrent competitors. State interventions,
deployment swaps, and granularity analyses test the persistent access layer
(\LinkSection{sec:experiments}).
Our contributions are:
\begin{itemize}[left=0pt]
\item \textbf{Memory Index Evolution.} We isolate persistent access state as an
adaptation target complementary to stored content, with value assessed on
subsequent queries (\LinkSection{sec:formulation}).
\item \textbf{Adaptive Memory Index Evolution.} We train a retrieval-conditioned
index policy by balanced mask distillation and couple its selective actions
to bounded, persistent score updates within native host interfaces
(\LinkSection{sec:method}).
\item \textbf{Empirical regularities.} Controlled analyses reveal how retrieval history, persistent suppression, and policy selectivity govern access adaptation, and identify the structural conditions under which it can expose unseen evidence without suppressing useful memories (\LinkSection{sec:experiments}).
\end{itemize}

\section{Problem Formulation}
\label{sec:formulation}

\paragraph{Retrieval-based Agent Memory.}
Let $\mathcal{C}_t=\{m_1,\ldots,m_{N_t}\}$ denote the experience retained
by an agent at time $t$. Given a query $q$, memory is useful only through
the downstream task it supports. We therefore define the ideal memory
support as
\begin{equation}
    \mathcal{S}^{*}(q)
    \in
    \operatorname*{arg\,max}_{\mathcal{S}\subseteq\mathcal{C}_t}
    \mathcal{U}\!\left(q,\mathcal{A}(\cdot\mid q,\mathcal{S})
    \right),
    \label{eq:memory-objective}
\end{equation}
where $\mathcal{A}$ is the downstream agent or generator and
$\mathcal{U}$ is task utility. The maximisation permits the full store as a
special case; more exposed content need not improve utility. Retained content
defines the available evidence, while the index determines what the host
surfaces for the agent.

\paragraph{Content and Index Evolution.}
We therefore distinguish the retained content $\mathcal{C}_t$ from the
\emph{memory index} $\mathcal{I}_t$, the state governing how that content
is surfaced:
\begin{equation}
    \mathcal{M}_t=(\mathcal{C}_t,\mathcal{I}_t),
    \qquad
    \mathcal{S}_t
    =
    \mathcal{R}(q_t;\mathcal{C}_t,\mathcal{I}_t)
    \subseteq\mathcal{C}_t .
    \label{eq:memory}
\end{equation}
Here $\mathcal{R}$ denotes the host retrieval procedure; our formulation
does not assume a particular implementation. The same memory objective
can then be improved along two complementary axes,
\begin{align}
    \underbrace{
    \mathcal{C}_{t+1}=G(\mathcal{C}_t,\mathcal{O}_t)
    }_{\text{content evolution}},
    \qquad
    \underbrace{
    \mathcal{I}_{t+1}=F(\mathcal{I}_t,\mathcal{O}_t)
    }_{\text{index evolution}},
    \label{eq:two-axis-evolution}
\end{align}
where $\mathcal{O}_t$ is information available from the runtime
interaction. $G$ learns \emph{what to retain}; $F$ learns \emph{how
retained experience should remain reachable}. We isolate the latter as a
separate update process learned from retrieval observations. Holding content
fixed defines the controlled intervention here; it does not preclude combining
index and content evolution.

\paragraph{Memory Index Evolution.}
Index evolution seeks an update rule $F$ that improves the same
downstream objective as \LinkEquation{eq:memory-objective} over a lifelong
query stream:
\begin{equation}
    F^\star
    =
    \operatorname*{arg\,max}_{F}
    \mathbb{E}\!\left[
        \sum_{t=1}^{T}
        \mathcal{U}\!\left(q_t,\,
            \mathcal{A}\!\left(\cdot\mid q_t,\,
                \mathcal{R}(q_t;\mathcal{C}_t,\mathcal{I}_t)
            \right)
        \right)
    \right],
    \qquad
    \mathcal{I}_{t+1}=F(\mathcal{I}_t,\mathcal{O}_t).
    \label{eq:index-objective}
\end{equation}
Because an update to $\mathcal{I}_t$ affects subsequent queries, its value
must be assessed prospectively rather than on the interaction that produced
it.
Index evolution poses three design challenges:

\begin{itemize}[left=0pt]
\item\phantomsection\label{ch:interference}
\textbf{C1: What signal reveals access interference?}
Interference depends on the queries a deployment actually issues and
cannot be determined from memory content alone.

\item\phantomsection\label{ch:prospective}
\textbf{C2: Which updates generalize forward?}
A persistent index update must benefit unseen queries rather than merely
reinforce the interaction that produced it.

\item\phantomsection\label{ch:inflow}
\textbf{C3: How can the index evolve in a running system?}
Index adaptation must operate online through the host's retrieval interface.
\end{itemize}

\section{Adaptive Memory Index Evolution}
\label{sec:method}

AMIE addresses the three challenges through coupled decisions about
\emph{what to update}, \emph{how the update persists}, and \emph{where it
acts}. A \emph{retrieval-conditioned index policy} identifies suppression
actions from deployment queries and retrieved memories, supplying the
interaction-grounded signal required by \hyperref[ch:interference]{C1}
(\LinkSection{sec:method-estimation}). \emph{Bounded persistent access
dynamics} carry these actions across queries, addressing the prospective
update required by \hyperref[ch:prospective]{C2}
(\LinkSection{sec:method-reweight}). \emph{In-the-flow index evolution}
applies this state within the host's scored candidate interface, addressing
\hyperref[ch:inflow]{C3} without rewriting stored content
(\LinkSection{sec:method-inflow}).

This design factors the update in \LinkEquation{eq:index-objective} into
a retrieval-conditioned policy and a persistent state transition:
\begin{equation}
F_\theta(\mathcal I_t,\mathcal O_t)
=\mathsf{T}\big(\mathcal I_t,\pi_\theta(\mathcal O_t)\big).
\label{eq:policy-factorization}
\end{equation}
Here $\pi_\theta$ selects per-memory actions and $\mathsf T$ converts them
into future-access state. Since future utility is unavailable when an action
is taken, teacher-mask distillation supplies a surrogate for fitting
$\pi_\theta$; subsequent-query utility evaluates the resulting state against
the prospective objective. \LinkFigure{fig:method} shows this separation
between offline policy learning and online index evolution.

\begin{figure}[t]
\centering
\begingroup
\definecolor{otblue}{HTML}{809BCE}
\definecolor{otorange}{HTML}{95B8D1}
\definecolor{otteal}{HTML}{B8E0D4}
\definecolor{otgray}{HTML}{D6EADF}
\definecolor{otpurple}{HTML}{EAC4D5}
\def\blockwidth{28mm}
\def\columnstep{3.48}
\def\rowstep{1.55}
\begin{tikzpicture}[
 x=1cm,y=1cm,font=\fontsize{8}{9}\selectfont,
 box/.style={draw=otgray!55!black,fill=otgray!35,rounded corners=2pt,
 align=center,inner xsep=2.5mm,inner ysep=2.5mm,
 minimum height=10mm,text width=\blockwidth},
 wrap/.style={box,draw=otblue!65!black,fill=otblue!30},
 state/.style={box,draw=otpurple!55!black,fill=otpurple!40},
 flow/.style={-{Stealth[length=1.4mm]},line width=.6pt},
 control/.style={flow,dashed,draw=otpurple!45!black},
 note/.style={font=\fontsize{7}{8}\selectfont,align=center}
]
\node[box] (native) at (0,0) {Native candidates\\scored pool $C_t$};
\node[wrap] (rank) at (\columnstep,0) {AMIE re-rank\\using $\mathbf u_t$};
\node[box] (seal) at ({2*\columnstep},0) {Seal answer context\\$\mathcal S_t$};
\node[box] (gen) at ({3*\columnstep},0) {Generator\\current answer};
\node[note] (q) at (0,1.15) {Query $q_t$};
\draw[flow] (q)--(native);
\draw[flow] (native)--(rank);
\draw[flow] (rank)--(seal);
\draw[flow] (seal)--(gen);
\node[state,minimum height=6mm,inner ysep=1.5mm]
 (pre) at (\columnstep,1.15) {Persistent state $\mathbf u_t$};
\draw[control] (pre)--(rank);
\node[note,text width=65mm] at ({2.5*\columnstep},1.15)
 {Offline: teacher masks train $\pi_\theta$.\\Online: $\theta$ fixed; access state evolves.};
\node[box] (fb) at (0,-\rowstep) {Native feedback\\$\mathcal B_t$};
\node[wrap] (policy) at (\columnstep,-\rowstep)
 {Index policy $\pi_\theta$\\query $q_t$ + feedback $\mathcal B_t$};
\node[state] (update) at ({2*\columnstep},-\rowstep)
 {Commit mask $\mathbf a_t$\\$\mathbf u_{t+1}=\mathsf T(\mathbf u_t,\mathbf a_t)$};
\node[box,draw=otteal!55!black,fill=otteal!45]
 (later) at ({3*\columnstep},-\rowstep)
 {Future retrieval\\uses $\mathbf u_{t+1}$};
\draw[flow] (native)--(fb);
\draw[flow] (fb)--(policy);
\draw[flow] (policy)--(update);
\draw[control] (seal)--node[right,note]{after sealing}(update);
\draw[control] (update)--(later);
\node[note,text width=\linewidth,inner xsep=0pt] at ({1.5*\columnstep},-2.4)
 {Selected competitors lose priority; useful in-pool memories can rise on later queries.};
\end{tikzpicture}
\endgroup
\caption{\textbf{AMIE applies persistent access state to subsequent retrievals.}
Native candidates supply both the answer context, ranked with pre-update
state, and the policy's feedback. Dashed arrows show state access and the
commit constraint: updates occur after context sealing and affect later queries.}
\label{fig:method}
\end{figure}
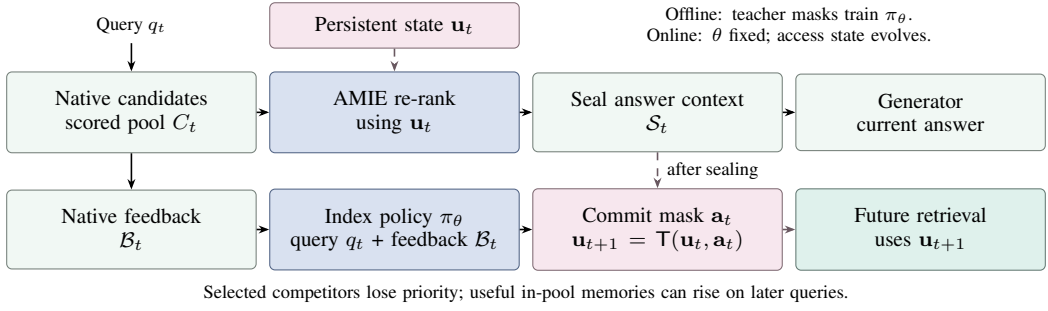

\subsection{Retrieval-Conditioned Index Policy}
\label{sec:method-estimation}
\label{sec:method-scope}

\paragraph{Predict selective updates from native retrieval feedback.}
The local decision is whether a surfaced memory deserves a persistent
penalty, which depends on its relation to the query. The policy therefore observes $q_t$ and
the host's ordered native feedback list $\mathcal B_t$. This list is distinct
from the answer context $\mathcal S_t$, which AMIE re-ranks using existing
state. The observation $\mathcal O_t$ contains the query and feedback memories'
text; future queries and reference answers are excluded. The action space is
\begin{equation}
\mathbf a_t=\pi_\theta(\mathcal O_t),\qquad
 a_{i,t}\in\{0,1\},\qquad
 a_{i,t}=0\ \text{for }m_i\notin\mathcal B_t.
\label{eq:policy-mask}
\end{equation}
An action of one increments a memory's suppression state; zero preserves
its current state. A keep action avoids additional suppression; it does not
remove an accumulated penalty or prevent changes in relative rank.

\paragraph{A language model predicts the action mask in one forward pass.}
The policy jointly encodes the query and ordered feedback memories to
predict a suppression probability $p_{\theta,i}(\mathcal O_t)$ for each
surfaced entry. Its binary decisions map back to stable memory IDs, so a
query-conditioned judgment becomes an update to a persistent object.
The complete mask is emitted in one pass, without per-entry textual
judgments; the backbone, prediction head, and fitting configuration are in
\LinkAppendix{app:selected-policy}.

\paragraph{Teacher masks supply a surrogate for selective state updates.}
\phantomsection\label{sec:method-trained}
A teacher labels feedback memories as keep or suppress,
favouring keep when uncertain about future usefulness
(\LinkFigure{fig:prompt-mask}). On an offline collection $\mathcal D$
of observations and teacher masks, we fit the student with
class-balanced binary cross-entropy:
\begin{equation}
\mathcal L(\theta)=\mathbb E_{(\mathcal O_t,\mathbf y_t)\sim\mathcal D}
\left[\frac{1}{|\mathcal B_t|}\sum_{i:m_i\in\mathcal B_t}
 w_{y_{i,t}}\,\ell_{\mathrm{BCE}}\big(p_{\theta,i}(\mathcal O_t),y_{i,t}\big)\right],
\label{eq:policy-imitation}
\end{equation}
where $y=1$ denotes suppress, $y=0$ denotes keep, and $w_y$ balances
the two classes. At deployment,
$\theta$ is fixed: adaptation occurs in the persistent index state.
Fitting details and benchmark exposure are in \LinkAppendix{app:selected-policy};
\LinkAppendix{app:prospective-value} defines the later-query evaluation criterion.

\paragraph{The action interface separates selectivity from persistence.}
\phantomsection\label{sec:method-count}
\label{sec:method-llm}
The frozen selector prompts the same backbone with the mask instruction,
without task-specific fitting. The count selector sets $a_{i,t}=1$ for every
native feedback entry, using repeated exposure as a proxy for interference
with no model calls. These alternatives share the trained policy's state
transition and host wrapper, isolating what learned selectivity adds
(\LinkSection{sec:exp-count-default}).

\subsection{Bounded Persistent Access Dynamics}
\label{sec:method-reweight}
\label{sec:method-policy}
\paragraph{Accumulate selected actions into bounded access state.}
We instantiate the persistent access state with bounded multiplicative
suppression. To realise the cross-query effect in
\LinkEquation{eq:index-objective}, AMIE accumulates actions by memory
identity: an observation changes a retained access variable that subsequent
queries can reuse. Its access state is
\begin{equation}
\mathcal I_t=\mathbf u_t=(u_{1,t},\ldots,u_{N,t}),\qquad u_{i,t}\in[0,1].
\label{eq:state}
\end{equation}
The underlying counters evolve as
\begin{equation}
n_{i,t}=\sum_{\tau<t}a_{i,\tau},\qquad
n_{i,t+1}=n_{i,t}+a_{i,t},\qquad n_{i,1}=0.
\label{eq:masked-update}
\end{equation}
Counts are converted to bounded suppression state,
\begin{equation}
u_{i,t}=g(n_{i,t})=\frac{n_{i,t}}{n_{i,t}+\kappa},\qquad \kappa>0,
\label{eq:shrinkage}
\end{equation}
where $\kappa$ sets the number of selected events needed for half-saturation.
Repeated decisions strengthen the penalty, while saturation bounds it.
The state summarises accumulated actions, not a probability of irrelevance.
It persists when a memory leaves the feedback list, allowing past competition
to influence its next retrieval without re-observing the original query.

\paragraph{Relative access can improve without observing the target.}
After a host-specific monotone affinity mapping and positive query-wise
scaling to $\hat\sigma(q,i)\in[0,1]$, AMIE uses
\begin{equation}
s_t(q,i)=\hat\sigma(q,i)\big(1-\lambda_p u_{i,t}\big),
\qquad \lambda_p\in[0,1].
\label{eq:score}
\end{equation}
The strength $\lambda_p$ bounds the score reduction. An unchanged memory can
rise in rank when its competitors are suppressed, even if the policy has
never observed that memory. \LinkProposition{prop:margin} formalises this
conditional rank property in \LinkAppendix{app:relative-access-proof}.
Selectivity matters because suppressing useful evidence can instead impair
its later access (\LinkAppendix{app:traffic-risk}).

\subsection{In-the-Flow Index Evolution}
\label{sec:method-inflow}

\paragraph{A shared wrapper makes access state independently controllable.}
AMIE attaches to stable memory IDs and the host's scored candidate interface,
separating the evolving access layer from content construction and candidate
generation:
\begin{equation}
\mathcal H:q\mapsto[(m_i,\sigma_{\mathcal H}(q,m_i))]_{i\in C_{\mathcal H}(q)},
\label{eq:host}
\end{equation}
and ranks the host-supplied candidate pool using the pre-update state:
\begin{equation}
\begin{aligned}
\mathcal W_{\mathbf u_t}(\mathcal H)(q)
 &=\operatorname{Top}_{k,\prec}\!\left[
 (m_i,\hat\sigma_{\mathcal H}(q,i)(1-\lambda_p u_{i,t}))
 \right]_{i\in C_{\mathcal H}(q)},\\
\mathcal W_{\mathbf u_t}(\mathcal H)(q)\big|_{\lambda_p=0}
 &=\operatorname{Select}_{k,\prec}(\mathcal H(q)).
\end{aligned}
\label{eq:wrapper}
\end{equation}
Both operators return ordered memory IDs at depth $k$, with the host's tie
order $\prec$. Setting $\lambda_p=0$ recovers the host's selection from the
same pool, enabling matched comparisons of the access layer. Monotone score
adaptation preserves the host order before state is applied; host-specific
attachment points are specified in \LinkAppendix{app:attachment}.

\paragraph{Post-sealing updates enforce the prospective adaptation target.}
For query $t$, the wrapper reads $\mathbf u_t$ and seals the answer context
$\mathcal S_t$ before committing the mask derived from native feedback
$\mathcal B_t$. Only later queries use $\mathbf u_{t+1}$, so an action's
effect is carried forward rather than credited to the answer that elicited
it. Native feedback also separates the update signal from AMIE's own
re-ranked context. Candidate-pool depth, feedback depth, and answer depth
therefore control reach, observation, and context respectively; only in-pool
memories can enter the answer context, and acquiring an expanded pool can
incur additional retrieval work. This interface makes the policy and state
replaceable while keeping the host's stored content and answerer fixed.

\section{Experiments}
\label{sec:experiments}

\subsection{Experimental Setup}
\label{sec:exp-setup}

\paragraph{Benchmarks.}
LoCoMo tests long-horizon conversational memory, GroupMem tests group memory
across channels, and MemBench tests personal memory
\citep{maharana2024locomo,groupmem2025,membench2025}.
We report F1 on LoCoMo and accuracy otherwise; retrieval units and evidence
mappings are defined in \LinkAppendix{app:metrics}, and dataset scopes in
\LinkTable{tab:dataset-stats}.

\paragraph{Memory Hosts.}
We evaluate BM25 and released Mem0, MemoryOS, A-MEM, and Memory-R2 systems,
covering lexical, dense, tiered, and linked retrieval interfaces.
\LinkTable{tab:attachment} specifies native interfaces and AMIE attachment points. All hosts share the same generator backbone: Qwen3-32B.

\paragraph{Index-Evolution Policies.}
Trained, frozen, and count share the persistent state and host wrapper,
differing in which feedback entries receive updates.
Trained and frozen use the same backbone; the former learns teacher masks and
the latter follows the mask prompt, while count requires no model calls.
Policies were fitted or selected on LoCoMo and fixed before GroupMem and
MemBench evaluation (\LinkAppendix{app:selected-policy}).

\paragraph{Implementation and Evaluation Protocol.}
On/off comparisons are configured to share content, candidate pools, query
order, context policy, the frozen answerer, and grading; each
AMIE update follows context sealing.
Pool, feedback, and output depths separately determine reach, update signal,
and answer context (\LinkAppendix{app:attachment}).
Main tables report host-level scores; controlled replays isolate state
effects, and paired intervals apply only to aligned per-question records
(\LinkAppendix{app:full-results}, \LinkAppendix{app:count}).

\subsection{Does Index Evolution Improve Memory Access and Downstream Utility?}
\label{sec:main}
\label{sec:exp-qwen3}

\begin{table}[t]
\caption{\textbf{Retrieval scores with and without AMIE} (\%, higher is better).
Subscripts subtract the displayed baseline; green/red intensity encodes
positive/negative magnitude, not significance.
Host-specific retrieval units are defined in \LinkAppendix{app:metrics}.}

\label{tab:retrieval-main}
\centering\footnotesize
\setlength{\tabcolsep}{1.8pt}
\renewcommand{\arraystretch}{1.08}
\begin{tabular*}{\linewidth}{@{\extracolsep{\fill}}llrrrr@{}}
\toprule
 & & & \multicolumn{3}{c}{\textbf{+AMIE}} \\
\cmidrule(l){4-6}
\textbf{Benchmark} & \textbf{Host} & \textbf{Host only} & \textbf{Trained} & \textbf{Frozen} & \textbf{Count} \\
\midrule
\multirow{5}{*}{\shortstack[l]{LoCoMo\\\citep{maharana2024locomo}}}
 & BM25~\citep{robertson2009bm25} & 51.94 & \amieCell{53.02}{+1.08} & \amieCell{53.21}{+1.27} & \amieCell{53.41}{+1.47} \\
 & Mem0~\citep{mem0} & 42.17 & \amieCell{44.42}{+2.25} & \amieCell{44.28}{+2.11} & \amieCell{45.03}{+2.86} \\
 & MemoryOS~\citep{kang2025memoryos} & 45.29 & \amieCell{45.47}{+0.18} & \amieCell{46.20}{+0.91} & \amieCell{45.90}{+0.61} \\
 & A-MEM~\citep{xu2025amem} & 63.27 & \amieCell{65.14}{+1.87} & \amieCell{64.73}{+1.46} & \amieCell{66.02}{+2.75} \\
 & Memory-R2~\citep{yan2026memoryr2faircreditassignment} & 65.92 & \amieCell{67.78}{+1.86} & \amieCell{66.57}{+0.65} & \amieCell{65.03}{-0.89} \\
\midrule
\multirow{5}{*}{\shortstack[l]{MemBench\\\citep{membench2025}}}
 & BM25~\citep{robertson2009bm25} & 83.94 & \amieCell{85.62}{+1.68} & \amieCell{84.71}{+0.77} & \amieCell{81.82}{-2.12} \\
 & Mem0~\citep{mem0} & 77.41 & \amieCell{79.26}{+1.85} & \amieCell{78.53}{+1.12} & \amieCell{79.17}{+1.76} \\
 & MemoryOS~\citep{kang2025memoryos} & 73.33 & \amieCell{75.18}{+1.85} & \amieCell{74.62}{+1.29} & \amieCell{74.01}{+0.68} \\
 & A-MEM~\citep{xu2025amem} & 85.83 & \amieCell{87.21}{+1.38} & \amieCell{86.74}{+0.91} & \amieCell{86.86}{+1.03} \\
 & Memory-R2~\citep{yan2026memoryr2faircreditassignment} & 79.36 & \amieCell{81.07}{+1.71} & \amieCell{80.42}{+1.06} & \amieCell{80.13}{+0.77} \\
\bottomrule
\end{tabular*}
\end{table}

\begin{table}[t]
\caption{\textbf{Downstream utility with and without AMIE} (\%, higher is better).
F1 is used for LoCoMo and accuracy for GroupMem and MemBench.
Subscripts and green/red shading encode signed percentage-point differences
from the displayed baselines, not significance; dashes mark inapplicable arms.}

\label{tab:paired-main}
\centering\footnotesize
\setlength{\tabcolsep}{1.8pt}
\renewcommand{\arraystretch}{1.08}
\begin{tabular*}{\linewidth}{@{\extracolsep{\fill}}llrrrr@{}}
\toprule
 & & & \multicolumn{3}{c}{\textbf{+AMIE}} \\
\cmidrule(l){4-6}
\textbf{Benchmark} & \textbf{Host} & \textbf{Host only} & \textbf{Trained} & \textbf{Frozen} & \textbf{Count} \\
\midrule
\multirow{6}{*}{\shortstack[l]{LoCoMo (F1)\\\citep{maharana2024locomo}}}
 & No memory & 0.00 & - & - & - \\
 & BM25~\citep{robertson2009bm25} & 17.73 & \amieCell{19.04}{+1.31} & \amieCell{18.40}{+0.67} & \amieCell{18.27}{+0.54} \\
 & Mem0~\citep{mem0} & 9.98 & \amieCell{12.44}{+2.46} & \amieCell{12.69}{+2.71} & \amieCell{12.85}{+2.87} \\
 & MemoryOS~\citep{kang2025memoryos} & 20.27 & \amieCell{20.97}{+0.70} & \amieCell{20.83}{+0.56} & \amieCell{20.77}{+0.50} \\
 & A-MEM~\citep{xu2025amem} & 16.54 & \amieCell{17.43}{+0.89} & \amieCell{17.17}{+0.63} & \amieCell{16.92}{+0.38} \\
 & Memory-R2~\citep{yan2026memoryr2faircreditassignment} & 23.73 & \amieCell{24.96}{+1.23} & \amieCell{24.26}{+0.53} & \amieCell{24.51}{+0.78} \\
\midrule
\multirow{6}{*}{\shortstack[l]{GroupMem (Acc.)\\\citep{groupmem2025}}}
 & No memory & 18.39 & - & - & - \\
 & BM25~\citep{robertson2009bm25} & 32.08 & \amieCell{33.72}{+1.64} & \amieCell{32.48}{+0.40} & \amieCell{31.01}{-1.07} \\
 & Mem0~\citep{mem0} & 20.79 & \amieCell{22.14}{+1.35} & \amieCell{21.64}{+0.85} & \amieCell{20.31}{-0.48} \\
 & MemoryOS~\citep{kang2025memoryos} & 35.72 & \amieCell{38.27}{+2.55} & \amieCell{37.18}{+1.46} & \amieCell{34.66}{-1.06} \\
 & A-MEM~\citep{xu2025amem} & 23.27 & \amieCell{25.91}{+2.64} & \amieCell{24.36}{+1.09} & \amieCell{23.84}{+0.57} \\
 & Memory-R2~\citep{yan2026memoryr2faircreditassignment} & 33.23 & \amieCell{34.81}{+1.58} & \amieCell{34.17}{+0.94} & \amieCell{33.62}{+0.39} \\
\midrule
\multirow{6}{*}{\shortstack[l]{MemBench (Acc.)\\\citep{membench2025}}}
 & No memory & 48.38 & - & - & - \\
 & BM25~\citep{robertson2009bm25} & 78.93 & \amieCell{80.66}{+1.73} & \amieCell{80.27}{+1.34} & \amieCell{78.29}{-0.64} \\
 & Mem0~\citep{mem0} & 79.17 & \amieCell{82.54}{+3.37} & \amieCell{81.83}{+2.66} & \amieCell{81.26}{+2.09} \\
 & MemoryOS~\citep{kang2025memoryos} & 74.17 & \amieCell{75.63}{+1.46} & \amieCell{74.84}{+0.67} & \amieCell{73.91}{-0.26} \\
 & A-MEM~\citep{xu2025amem} & 81.67 & \amieCell{82.62}{+0.95} & \amieCell{81.36}{-0.31} & \amieCell{81.48}{-0.19} \\
 & Memory-R2~\citep{yan2026memoryr2faircreditassignment} & 79.36 & \amieCell{81.74}{+2.38} & \amieCell{80.92}{+1.56} & \amieCell{79.83}{+0.47} \\
\bottomrule
\end{tabular*}
\end{table}

\paragraph{Index evolution improves downstream utility across heterogeneous hosts.}
Trained AMIE improves all $\amieUtilityWins{}$ benchmark--host utility
comparisons, spanning the three memory settings and five retrieval systems
(\LinkTable{tab:paired-main}, \LinkFigure{fig:ga-summary}).
Aligned LoCoMo Mem0 and MemoryOS records also yield positive paired intervals;
the other host comparisons are point estimates
(\LinkAppendix{app:full-results}).

\paragraph{Access gains extend to the tasks that memory serves.}
All $\amieRetrievalWins{}$ pairs with retrieval measurements improve both
retrieval and downstream utility under trained AMIE
(\LinkTable{tab:retrieval-main}, \LinkTable{tab:paired-main}).
The agreement concerns aggregate metrics, not question-level mediation.

\paragraph{The depth diagnostic locates larger access gains near the list head.}
The BM25 count replay shows larger recall improvements at shallow retrieval
depths and a smaller estimate when much more of the candidate pool is returned
(\LinkTable{tab:depth}).
Improved ranking quality corroborates this redistribution within the same
pool; depth-specific intervals and the MemoryOS comparison are in
\LinkAppendix{app:count}. The sweep measures retrieval, not downstream utility.

\subsection{What Does Persistent Access State Learn from Retrieval History?}
\label{sec:exp-why}
\label{sec:exp-mechanism}

\paragraph{Observed competitors can expose evidence the policy has never seen.}
Counting improves recall for unexposed gold evidence in the LoCoMo BM25
replay, while label-informed competitor injection independently demonstrates
the relative-access mechanism in \LinkProposition{prop:margin}
(\LinkFigure{tab:gold-unseen}, \LinkFigurePanel{fig:replay-one}{a}).
The target need not be observed or promoted: previously observed competitors
can be demoted around it, whereas exposure of the target itself introduces
suppression risk (\LinkAppendix{app:injection}, \LinkAppendix{app:exposure-map}).

\paragraph{The gain resides in persistent, item-specific access state.}
Permuting penalties or reversing their direction removes the BM25 count
advantage at matched mean score movement
(\LinkFigurePanel{fig:evolvability-main}{c}).
A state frozen after even a small fraction of past queries improves future
retrieval (\LinkTable{tab:sufficiency}), and category-matched state
outperforms swapped state over the same store
(\LinkFigurePanel{fig:replay-one}{b,c}).
Reset controls, the non-monotone history-length curve, and the less decisive
interleaved swap qualify these findings in \LinkAppendix{app:count},
\LinkAppendix{app:sufficiency}, and~\LinkAppendix{app:swap}.

\paragraph{Suppression reallocates access within the host's candidate pool.}
The gains arise by moving existing candidates into the returned context:
the largest count gains occur where missing evidence remains reachable
inside the pool (\LinkFigurePanel{fig:evolvability-main}{b}).
Candidate coverage therefore bounds this mechanism, and an admitted memory
can displace useful evidence; label-assisted headroom and oracle interventions
quantify these opportunities and trade-offs
(\LinkAppendix{app:evolvability}, \LinkAppendix{app:ceilings}).

\begin{figure}[t]
\centering
\includegraphics[width=\linewidth]{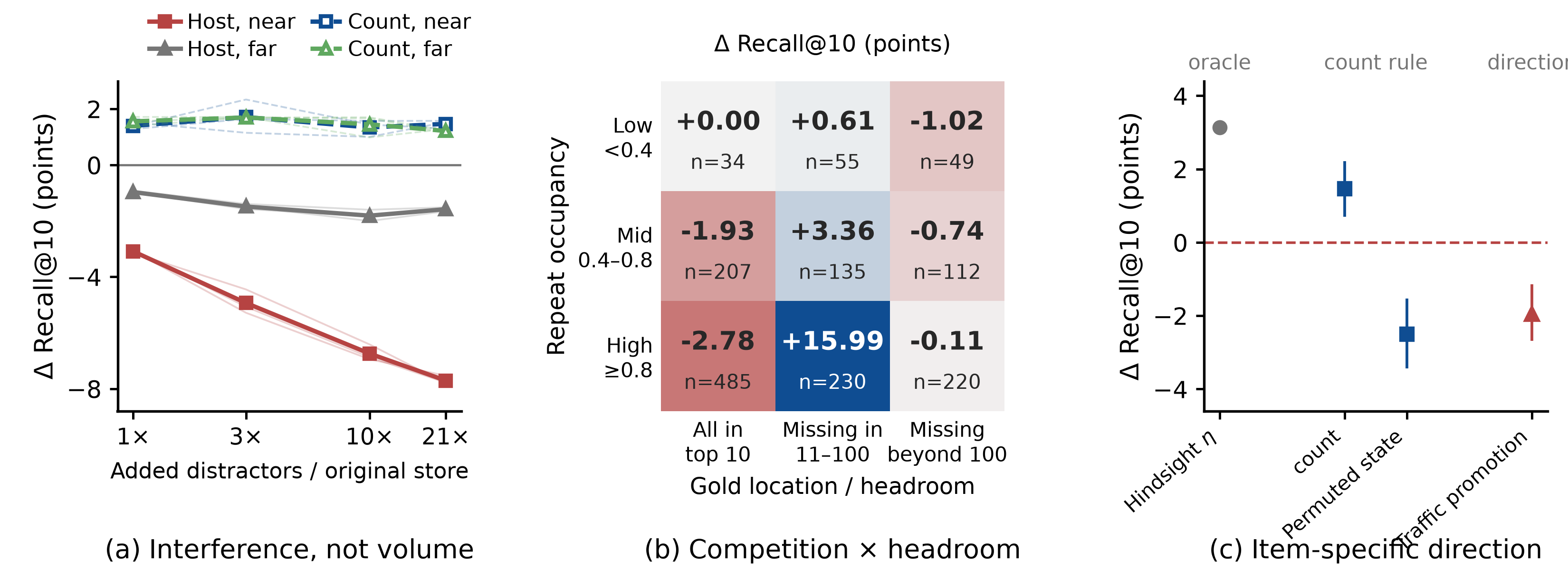}
\caption{\textbf{Competition and state interventions in the BM25 replay.}
Panels compare (a) added-store content, (b) competition and reachable headroom,
and (c) oracle and direction controls.
Recall-point differences cover $1{,}527$ questions, with conversation-cluster
$95\%$ intervals where records exist; baselines and controls are defined in
\LinkAppendix{app:count} and~\LinkAppendix{app:volume}.}
\label{fig:evolvability-main}
\end{figure}

\subsection{How Much Policy Selectivity Is Necessary?}
\label{sec:exp-count-default}

\paragraph{Retrieval traffic alone supplies a strong adaptation signal.}
Count approaches frozen-policy F1 on LoCoMo BM25 and MemoryOS without model
calls (\LinkTable{tab:paired-main}).
In the BM25 observation ablation, suppression rate tracks substantial
variation in retrieval gain ($r=0.79$), while differences between observation
regimes are smaller than within-regime seed variation
(\LinkTable{tab:student-grid}, \LinkAppendix{app:policy-ablations}).
The gain retained by a rate-matched random mask further supports accumulation
as a shared mechanism before semantic selection
(\LinkFigurePanel{tab:mechanism-controls}{a}).

\paragraph{Selectivity helps preserve repeatedly useful memories.}
MemoryOS exposure strata show that semantic selection trades some
unseen-evidence gain for less damage at high evidence exposure
(\LinkFigure{tab:gold-unseen}); individual strata retain uncertainty
(\LinkAppendix{app:exposure-map}).
At host level, trained selection exceeds count by $3.61$ points on GroupMem
MemoryOS, where count loses utility (\LinkTable{tab:paired-main}).
This demonstrates the practical value of selection; GroupMem recurrence was
not measured, so the replay does not identify the cause of that host's difference.

\paragraph{Learned selection has lower recorded latency than prompted selection.}
The student emits one mask per observation, with a median recorded policy
time of about $87$ ms in the Mem0 runs (\LinkTable{tab:costs}).
The frozen API policy takes longer in those logs, while count makes no model
calls; serving setups differ and timings exclude answer generation
(\LinkAppendix{app:reranker}).

\subsection{When Does Index Evolution Help?}
\label{sec:exp-boundary-main}
\label{sec:exp-scale}

\paragraph{Interference, rather than store volume alone, creates room for adaptation.}
Adding near-domain, far-domain, or duplicated content at matched volumes
produces different host degradation and count-policy effects in the BM25
replay (\LinkFigurePanel{fig:evolvability-main}{a}).
Confusable additions sustain competition, while self-duplication couples
competitors to useful evidence; complete volume and seed sweeps are in
\LinkAppendix{app:volume}.

\paragraph{Repeated competition and reachable headroom concentrate access gains.}
The largest BM25 count gain occurs when recurring competitors coexist with
missing evidence still inside the candidate pool
(\LinkFigurePanel{fig:evolvability-main}{b}).
The label-assisted strata in \LinkAppendix{app:evolvability} distinguish
this opportunity from already-satisfied retrieval or out-of-pool evidence;
occupancy alone correlates differently with gain on BM25 and MemoryOS
(\LinkAppendix{app:count}). These are replay diagnostics, not an unlabelled
online selection rule.

\paragraph{Coarse retrieval units can couple suppression to useful evidence.}
Re-indexing the same conversations into progressively coarser units lowers
evidence purity and reduces the count advantage toward zero
(\LinkTable{tab:granularity-ladder}).
The session-level interval spans zero, identifying separability as a risk
factor rather than a universal requirement for fine-grained storage
(\LinkAppendix{app:granularity-ladder}).

\paragraph{Retaining history and discounting stale actions are complementary.}
In the count replays, gradual decay improves recall over unlimited
accumulation at suitable half-lives on both hosts, while repeated resets
discard useful state (\LinkFigurePanel{tab:mechanism-controls}{b,c}).
These controls suggest preserving a useful history while limiting its
long-term suppression cost; decay is a diagnostic variant, not a component
of the main-table policies (\LinkAppendix{app:count}).

\paragraph{Deployment diagnostics distinguish structural reach from empirical risk.}
Future retrieval and candidate coverage delimit reach; recurrence,
granularity, and state lifetime are empirical risk factors
(\LinkTable{tab:conditions}). The strength sweep similarly identifies
beneficial settings and damage from aggressive suppression
(\LinkAppendix{app:lambda}), supporting deployment diagnostics rather than
universal conditions (\LinkAppendix{app:boundary}).

\section{Related Work}
\label{sec:related}

\subsection{Adaptive Agent Memory Systems}
Mem0, MemoryOS, and A-MEM extract, organise, and retrieve retained experience
\citep{mem0,kang2025memoryos,xu2025amem}; Memory-R1 and AtomMem learn memory
operations \citep{yan2025memoryr1,atommem2026}. Graph-structured variants
organise memories as linked entities and temporal knowledge graphs
\citep{graphrag,graphiti2024,gutierrez2024hipporag}, and multi-user settings
share memory under dynamic access control \citep{gmemory,collaborativememory}.
Persistent access adaptation also predates AMIE: MemoryBank updates
recall-dependent strength, Generative Agents maintains access recency, and
ExpGraph learns persistent node utilities within an evolving graph
\citep{zhong2024memorybank,park2023generative,expgraph2026}. Protecting
earlier behaviour under later adaptation is also the concern of continual
learning \citep{kirkpatrick2017ewc}; AMIE's analogue protects repeatedly
useful memories rather than model weights. Surveys taxonomise this design
space \citep{memsurvey2026,graphmemsurvey2026,guo2024massurvey}.
AMIE isolates index updates for separately controlled plug-in evaluation. \LinkAppendix{app:related-table}
maps the update targets of the compared systems.

\subsection{Adaptive Retrieval and Index Learning}
Retrieval-augmented generation couples a parametric reader with a
non-parametric index \citep{lewis2020rag,borgeaud2022retro,gao2023ragsurvey}.
Dense retrieval, cross-encoder reranking, and adaptive or agentic retrieval
improve which information a query receives
\citep{karpukhin2020dense,nogueira2019passage,asai2023selfrag,jiang2023flare,li2026beyond}.
Diversity- and fusion-based rerankers reorder a fixed candidate pool without
carrying state across queries \citep{carbonell1998mmr,cormack2009rrf}, while
index structures and learned indexes accelerate access without changing which
items are relevant \citep{malkov2018hnsw,jegou2011pq,kraska2018case}.
AMIE carries per-memory state from one retrieval event to later
queries, through the host's existing candidate interface. The distinction is
the independently controlled update and its prospective
effect on subsequent queries.

\subsection{Learning from Retrieval Feedback}
A complementary line optimises memory behaviour from deployment feedback.
Memory-R1 and Memory-R2 train memory operations with reinforcement learning
over task outcomes \citep{yan2025memoryr1,yan2026memoryr2faircreditassignment}; Mem-$\alpha$
learns memory construction via reinforcement learning \citep{wang2025memalpha};
indexed experience memory supports long-horizon agents through retrieval over
prior trajectories \citep{memexrl2026}; and proactive or unified managers
learn when to write, retrieve, or consolidate
\citep{agemem,promem2026,langmem2025}. Benchmarks increasingly probe such
behaviour through incremental multi-turn interactions
\citep{hu2025memoryagentbench}. AMIE also learns from deployment feedback, but
at a different target: the index state over content that remains fixed, with
teacher judgments distilled into a policy whose actions persist across
subsequent retrievals.

\section{Conclusion}
\label{sec:con}

AMIE establishes the memory index as a persistent adaptation layer distinct
from stored content, distilling interference judgments into selective actions
whose effects persist across subsequent host retrievals. Across heterogeneous
memory hosts, AMIE consistently improves downstream utility, while controlled
interventions show that suppressing persistent competitors can expose evidence
the policy has never observed. Strikingly, a zero-call count policy is already
competitive with semantic policies in several settings, while learned
selectivity extends the benefit where unconditional suppression harms useful
memories. Together, these results establish how memories are accessed as a
practical adaptation axis alongside what an agent chooses to retain.

\bibliography{iclr2027_conference}
\bibliographystyle{iclr2027_conference}

\clearpage

\appendix
\counterwithin{table}{section}
\counterwithin{figure}{section}
\renewcommand{\theHtable}{appendix.\thesection.\arabic{table}}
\renewcommand{\theHfigure}{appendix.\thesection.\arabic{figure}}

\section{Experiment}
\subsection{Datasets}
\label{app:datasets}

The evaluated benchmarks differ in whether and how a deployment revisits
the same memory store. \LinkTable{tab:dataset-stats} summarises them.

\begin{table}[htbp]
\caption{\textbf{Benchmark statistics and evaluation scope.}
Counts describe the benchmark-level utility and evidence-valid scopes,
not completed sample sizes for every host; retrieval requires resolvable
supporting-evidence identifiers.
Revisiting describes repeated access to a store, not a measured guarantee
of competitor recurrence.}
\label{tab:dataset-stats}
\centering\small
\setlength{\tabcolsep}{4pt}
\resizebox{\linewidth}{!}{%
\begin{tabular}{lrrlll}
\toprule
\textbf{Benchmark} & \textbf{Benchmark scope (util.\slash retr.)} & \textbf{Memory units} & \textbf{Setting} & \textbf{Metric} & \textbf{Revisiting} \\
\midrule
LoCoMo & 1,540 / 1,527 & 5,882 turns (10 conv.) & conversational & F1 & yes \\
GroupMem & 745 / -- & 120,000 messages (33 channels) & cross-channel workplace & Acc. (GPT-5) & yes \\
MemBench & 23,137 / 23,132 & ${\sim}1.2$M memories & personal & Acc. (MCQ) & history retrieval \\
\bottomrule
\end{tabular}}

\end{table}

\paragraph{LoCoMo~\citep{maharana2024locomo}.} Ten long-horizon conversations between two speakers,
stored as $5{,}882$ dialogue turns. We evaluate the $1{,}540$ answerable
category~1--4 questions and score retrieval on the $1{,}527$ whose nonempty
evidence identifiers resolve to source turns; the adversarial category is
excluded, as in the official protocol. Gold evidence is the set of annotated
supporting turns. Questions are presented in the released order within each
conversation, and the state persists across the whole conversation. Because
many questions share supporting turns, LoCoMo is the benchmark where earlier
traffic informs later queries, and also the one where uncontrolled evaluation
can credit an operator for memorising shared evidence, which is why we
partition by gold exposure throughout (\LinkFigure{tab:gold-unseen}).

\paragraph{GroupMem~\citep{groupmem2025}.} A group-conversation memory benchmark of $120{,}000$
workplace and technical messages across $33$ channels in four domains
(Finance, Technology, Manufacturing, Healthcare), with $745$ questions.
The benchmark releases no supporting-message ground truth and is evaluated
on downstream utility only. Each domain forms one shared store serving
$145$--$214$ questions. Recurrence was not measured, so the policy-dependent
scores do not identify why counting degrades on some hosts.

\paragraph{MemBench~\citep{membench2025}.} Each trace has one terminal
multiple-choice question. In the BM25 chronological adapter, state starts at
zero and resets per trace. Before ingesting an incoming message, the adapter
retrieves up to ten earlier occurrences using the incoming user text
(participation) or message text (observation), then updates their state.
Empty queries and empty prefixes produce no retrieval event. These are
explicit adapter-generated history queries, not retrieval events required
by the benchmark. Memory IDs are trace-scoped; repeated source IDs receive
occurrence suffixes. The terminal context uses only the accumulated history
state and is sealed before any terminal-query update. Of $23{,}137$ traces,
$23{,}132$ have resolvable gold source IDs. This documented adapter should
not be inferred from history ingestion alone or assumed for other hosts
without their event records.

\subsection{Metrics}
\label{app:metrics}

\paragraph{Retrieval.} Source-evidence recall averages
$|G_q\cap P(S_q)|/|G_q|$ over evidence-valid questions, where $G_q$ contains
distinct annotated source units and $P(S_q)$ projects returned entries to
those units. The returned unit and depth are separate from the scoring unit.
BM25 on LoCoMo returns ten turns. MemoryOS returns up to seven pages covering
at most fourteen turns in total; recall counts the union of covered source
turns before section truncation. A-MEM's root metric scores ten selected
note roots before linked expansion, where their source identities are known.
For BM25 MemBench, ten occurrence slots project to distinct original message
IDs: repeated IDs consume slots but contribute at most once to coverage.
Questions with empty gold or any unresolved gold ID are excluded from recall,
not from the state trajectory or utility evaluation. GroupMem has no
supporting-message gold and is evaluated only for utility.

For generated entries, source-evidence recall requires an explicit
entry-to-source projection, including one-to-many mappings and the treatment
of unmapped entries. The entry-to-source projection for the Mem0 and Memory-R2 scores in
\LinkTable{tab:retrieval-main} remains unspecified here; reproducing these
scores as source-evidence recall requires that mapping. This boundary does not
affect answer-based utility scoring.

\paragraph{Downstream utility.} LoCoMo uses the official category-specific
stemmed token F1; MemBench uses exact equality with the reference choice;
GroupMem accuracy uses its default GPT-5 judge. All answers within a benchmark are
produced by one frozen Qwen3-32B answerer under identical decoding, and the
no-memory control uses the same answerer with no supplied memories.

\paragraph{Uncertainty.} Main-table subscripts subtract the displayed
baseline; colour encodes signed magnitude. Paired intervals use per-question
differences and $10{,}000$ bootstrap draws, clustered by conversation for
LoCoMo and by question for MemBench. Question-order studies report $95\%$
$t$ intervals over permutations. Main-table cells without aligned
per-question outputs carry no significance claim.

\paragraph{Matched replay.} The mechanism studies reproduce the BM25
LoCoMo count pair in \LinkTable{tab:retrieval-main}, $51.94\to53.41$, over
$1{,}527$ evidence-valid questions in the released order. They vary state
or retrieval conditions while holding the baseline protocol fixed.
Its mean unrounded difference is $+1.48$ points, versus $+1.47$ from
subtracting the displayed scores. MemoryOS exposure analyses use the native
seven-page retrieval pair, $45.29\to45.90$; these are recall over source
turns covered by pages, not a ten-turn context.

\subsection{Memory System Hosts}
\label{app:memory-hosts}
The control category contains answer generation without supplied memory; BM25
is the transparent sparse retriever. For each released host, we preserve its
native store, candidate generator, score semantics, and expansion rules.
\begin{itemize}[leftmargin=*]
\item \textbf{Mem0}~\citep{mem0} extracts candidate facts from interactions and
uses memory operations to add, update, delete, or retain stored records before
retrieving them by semantic similarity.
\item \textbf{MemoryOS}~\citep{kang2025memoryos} organizes personalized memory
across short-, mid-, and long-term tiers and manages consolidation, updating,
and retrieval through its conversational interface.
\item \textbf{A-MEM}~\citep{xu2025amem} builds Zettelkasten-inspired structured
notes with generated attributes and links, then retrieves note roots with
native linked expansion.
\item \textbf{Memory-R2}~\citep{yan2026memoryr2faircreditassignment} jointly
trains fact extraction and memory management, using local and global
group-relative objectives to assign credit across long interaction horizons.
\end{itemize}
AMIE changes only persistent access weights over identifiable candidates
returned by these host pipelines. The current comparison includes Mem0,
MemoryOS, A-MEM, and Memory-R2 alongside BM25 on LoCoMo, GroupMem, and MemBench;
GroupMem has utility results only.

\subsection{Native indexes, and where AMIE attaches to them}
\label{app:attachment}

The documented adapters expose scored candidates with stable memory IDs.
The native feedback list is retained independently of re-ranking. When its
entries carry suppression, the wrapper can extend the ranking pool to at
most $100$ candidates, preserving the native list first as a tie anchor.
Mem0 and A-MEM request this larger pool through their search interfaces;
MemoryOS reuses pages from the same segment search without a second stateful
retrieval. Disabled-state checks reproduce native IDs and order, but do not
by themselves establish matched candidate-acquisition cost or a same-expanded-pool
no-op comparison. \LinkTable{tab:attachment} lists four documented
interfaces; \LinkFigure{fig:attachment} illustrates these examples.

The index has a different concrete representation in each host. BM25 uses
turn-level term frequencies and lengths to score query-term matches.
Mem0 stores extracted fact vectors with record and user identifiers;
our vector-only configuration retrieves user-filtered facts by squared
$L_2$ distance~\citep{mem0}. MemoryOS indexes topic-summary vectors and
page vectors within the selected mid-term segments, then merges retrieved
pages with short-term history and long-term profile/knowledge
context~\citep{kang2025memoryos}. Its released search uses normalized
inner products, with segment and page gates of $0.1$; the $0.6$ threshold
belongs to segment merging during construction, not query-time retrieval.
A-MEM embeds enriched notes (content, context, keywords and tags) and stores
links between notes~\citep{xu2025amem}. Our adapter preserves query
rewriting and cosine root selection, invokes the released linked-context
renderer on the selected roots, and then applies our $1{,}200$-word budget.
AMIE therefore changes access through persistent per-ID weights on these
candidate lists; it does not rebuild any host's underlying index.

\begin{figure}[t]
\centering

\resizebox{\linewidth}{!}{%
\begin{tikzpicture}[
  font=\fontsize{8}{9}\selectfont, >={Stealth[length=1.5mm]},
  nat/.style={draw=black!50, fill=black!6, rounded corners=2pt,
    align=center, inner sep=2pt, text width=#1, minimum height=10mm},
  wrap/.style={nat=#1, draw=blueC, fill=blueC!9, thick},
  state/.style={nat=#1, draw=redC, fill=redC!8, minimum height=7mm},
  note/.style={font=\fontsize{7}{8}\selectfont, text=black!70, align=center},
  flow/.style={->, line width=0.6pt, draw=black!65},
  fb/.style={->, line width=0.6pt, draw=redC, dashed},
]
\definecolor{blueC}{HTML}{0F4D92}
\definecolor{redC}{HTML}{B64342}
\node[anchor=west, inner sep=0pt] at (0,0.48)
  {\textcolor{black!70}{Grey: native index / retrieval}\quad
   \textcolor{blueC}{Blue: AMIE re-ranking}\quad
   \textcolor{redC}{Red: feedback / persistent state}};
\foreach \x/\y/\title/\unit/\depth in {
  0/0/{(a) BM25}/{turn}/10,
  8.3/0/{(b) Mem0}/{fact}/10,
  0/-5.8/{(c) MemoryOS}/{page}/7,
  8.3/-5.8/{(d) A-MEM}/{note}/10}{
  \begin{scope}[shift={(\x,\y)}]
    \draw[draw=black!55, dashed, rounded corners=4pt] (0,0) rectangle (7.9,-5.4);
    \node[anchor=west, font=\small\bfseries] at (0.18,-0.32) {\title};
    \node[wrap=14mm] (w) at (4.38,-2.65) {AMIE\\re-rank\\$k{=}\depth$};
    \node[state=20mm] (p) at (1.95,-4.05) {Policy $\pi$\\$(q,\mathrm{top}\text{-}k)$};
    \node[state=14mm] (u) at (4.38,-4.05) {\unit{} state\\$u_i$};
    \draw[fb] (p) -- node[above, note, text=redC] {mask} (u);
    \draw[fb] (u) -- node[right, note, text=redC, align=left] {next\\queries} (w);
  \end{scope}
}
\begin{scope}
  \node[nat=70mm, minimum height=11mm] (ai) at (3.95,-1.15)
    {\textbf{Index: terms $\mapsto$ turn IDs + term statistics}\\
     Tokenize turns; store term counts and turn lengths.};
  \node[nat=20mm] (ar) at (1.95,-2.65) {Query terms\\BM25 scores};
  \node[nat=19mm] (ao) at (6.55,-2.65) {10 turns\\answer context};
  \node at (0.35,-2.65) {$q$};
  \draw[flow] (0.57,-2.65) -- (ar);
  \draw[flow] (ai.south -| ar.north) -- (ar.north);
  \draw[flow] (ar) -- (3.60,-2.65);
  \draw[flow] (5.16,-2.65) -- (ao);
  \draw[fb] (ar) -- node[left, note, align=right] {native\\top-10} (1.95,-3.65);
  \node[note] at (3.95,-4.95) {AMIE reweights turn IDs; stored text and term statistics stay fixed.};
\end{scope}
\begin{scope}[shift={(8.3,0)}]
  \node[nat=70mm, minimum height=11mm] (bi) at (3.95,-1.15)
    {\textbf{Index: fact embeddings + fact IDs + user IDs}\\
     Extract facts; ADD / UPDATE / DELETE / NOOP.};
  \node[nat=20mm] (br) at (1.95,-2.65) {Embed query\\user-filtered\\$L_2$ search};
  \node[nat=19mm] (bo) at (6.55,-2.65) {10 facts\\answer context};
  \node at (0.35,-2.65) {$q$};
  \draw[flow] (0.57,-2.65) -- (br);
  \draw[flow] (bi.south -| br.north) -- (br.north);
  \draw[flow] (br) -- (3.60,-2.65);
  \draw[flow] (5.16,-2.65) -- (bo);
  \draw[fb] (br) -- node[left, note, align=right] {native\\top-10} (1.95,-3.65);
  \node[note] at (3.95,-4.95) {AMIE reweights fact IDs; vectors, records and user filters stay fixed.};
\end{scope}
\begin{scope}[shift={(0,-5.8)}]
  \node[nat=70mm, minimum height=11mm] (ci) at (3.95,-1.15)
    {\textbf{Index: topic-summary vectors $\to$ page vectors}\\
     Short-term FIFO $\to$ mid-term; heat-based tier updates.};
  \node[nat=20mm] (cr) at (1.95,-2.65) {Embed query\\match segments\\then pages};
  \node[nat=19mm] (co) at (6.55,-2.65) {7 pages\\context merge};
  \node at (0.35,-2.65) {$q$};
  \draw[flow] (0.57,-2.65) -- (cr);
  \draw[flow] (ci.south -| cr.north) -- (cr.north);
  \draw[flow] (cr) -- (3.60,-2.65);
  \draw[flow] (5.16,-2.65) -- (co);
  \draw[fb] (cr) -- node[left, note, align=right] {native\\top-7} (1.95,-3.65);
  \node[nat=59mm, minimum height=5mm, font=\scriptsize] (cb) at (3.45,-4.95)
    {Short-term history + long-term profile / retrieved knowledge};
  \draw[flow] (cb.east) -| (co.south);
\end{scope}
\begin{scope}[shift={(8.3,-5.8)}]
  \node[nat=70mm, minimum height=11mm] (di) at (3.95,-1.15)
    {\textbf{Index: enriched-note embeddings + note links}\\
     Generate context, keywords, tags; link and evolve notes.};
  \node[nat=20mm] (dr) at (1.95,-2.65) {Rewrite, embed\\cosine search};
  \node[nat=19mm, dashed] (dout) at (6.55,-2.65) {10 roots\\expand links\\answer context};
  \node at (0.35,-2.65) {$q$};
  \draw[flow] (0.57,-2.65) -- (dr);
  \draw[flow] (di.south -| dr.north) -- (dr.north);
  \draw[flow] (dr) -- (3.60,-2.65);
  \draw[flow] (5.16,-2.65) -- (dout);
  \draw[fb] (dr) -- node[left, note, align=right] {native\\top-10} (1.95,-3.65);
  \node[note, text width=71mm, text=redC] at (3.95,-4.95)
    {Root selection changes which linked notes enter the context;\\
     the 1,200-word cap is our evaluation budget.};
\end{scope}
\node[anchor=north, align=center, inner sep=0pt] at (8.1,-11.5)
  {\textcolor{blueC}{AMIE: $s_i=\hat\sigma_i(q)(1-\lambda_p u_i)$;
   re-sort up to 100 native candidates, return top-$k$.}\\
   \textcolor{redC}{Policy mask updates per-ID counts $n_i$;
   $u_i=n_i/(n_i+\kappa)$ persists across queries.}};
\end{tikzpicture}%
}

\caption{\textbf{Four native-index attachment examples.}
Grey modules retain host behaviour, blue modules reweight candidates, and red
feedback updates per-memory state for later queries.
The integration settings are specified in \LinkTable{tab:attachment}.}
\label{fig:attachment}
\end{figure}
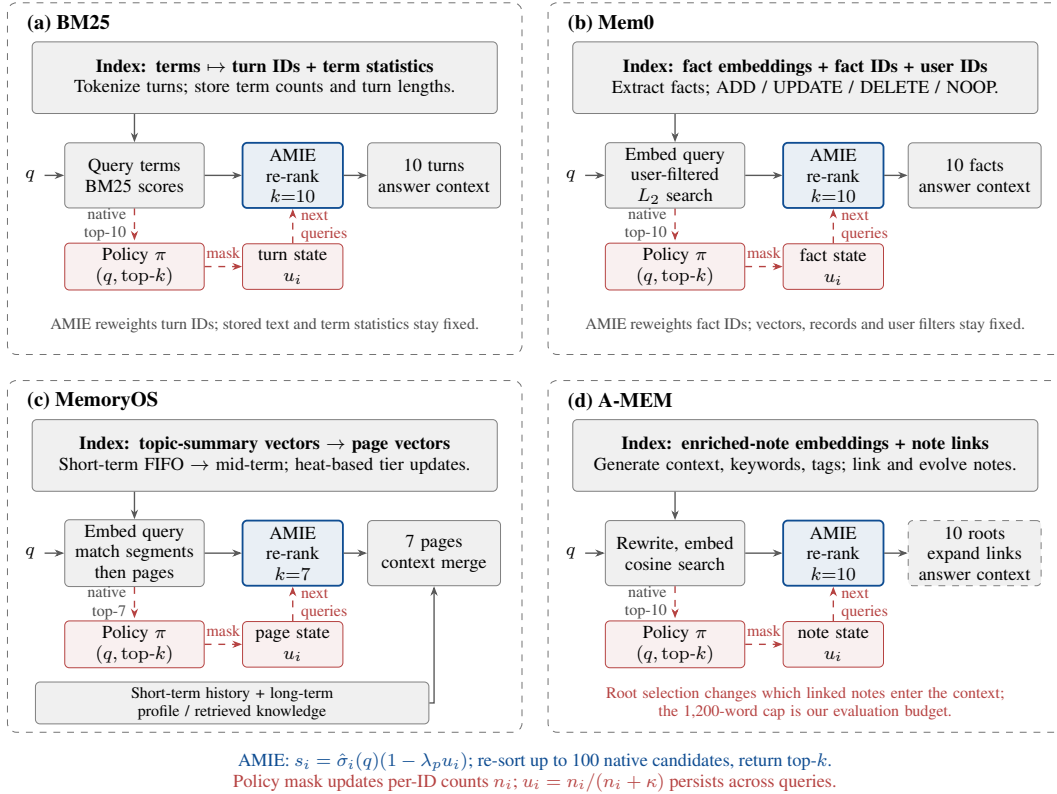

\begin{table}[htbp]
\caption{\textbf{Host interfaces for persistent access updates.}
AMIE reweights the scored candidates at the listed attachment point;
subsequent native expansion can propagate a root's suppression to linked evidence.}
\label{tab:attachment}
\centering\footnotesize
\setlength{\tabcolsep}{3pt}
\renewcommand{\arraystretch}{1.15}
\begin{tabular}{@{}>{\raggedright\arraybackslash}p{.11\linewidth}>{\raggedright\arraybackslash}p{.24\linewidth}>{\raggedright\arraybackslash}p{.20\linewidth}>{\raggedright\arraybackslash}p{.21\linewidth}>{\raggedright\arraybackslash}p{.16\linewidth}@{}}
\toprule
\textbf{Host} & \textbf{Native index / unit} & \textbf{Score adapter} & \textbf{Attachment / depth} & \textbf{Following stage} \\
\midrule
BM25 & Inverted index / turn & Nonnegative BM25 & Scored turn list / $10$ & Turn delivery \\
Mem0 & Fact vectors and user filter / fact & Squared distance $d\mapsto1/(1+d)$ & Filtered fact list / $10$ & Fact delivery \\
MemoryOS & Topic and page vectors / page & Inner product; segment and page gates $0.1$ & Scored pages / $7$ & Merge history, profile, knowledge \\
A-MEM & Enriched-note vectors and links / root & Cosine $s\mapsto(1+s)/2$ & Root list / $10$ & Link expansion; $1{,}200$-word cap \\
\bottomrule
\end{tabular}
\end{table}

The attachment unit determines what one action can affect. Suppressing a
Mem0 fact changes access to that record; suppressing an A-MEM root can also
remove linked evidence reached through that root. The host's native score
calibration likewise sets how strongly a given weight changes its ranking.
These interface properties motivate selective updates without assigning a
common failure mechanism to the host-level scores.

\begin{table}[t]
\caption{\textbf{Policies sharing the AMIE action and state interface.}
Representative effects follow \LinkTable{tab:retrieval-main} and~\LinkTable{tab:paired-main};
serving costs follow \LinkTable{tab:costs}.}
\label{tab:arms}
\centering\small
\setlength{\tabcolsep}{4pt}
\resizebox{\linewidth}{!}{%
\begin{tabular}{llll}
\toprule
\textbf{Policy} & \textbf{Cost / query} & \textbf{Representative effect} & \textbf{Role in the model} \\
\midrule
Trained & $87$ ms & MemoryOS F1 $+0.70$; Mem0 F1 $+2.46$
 & Learned selective actions \\
Frozen & $24$ s & MemoryOS recall $+0.91$
 & Direct semantic mask inference \\
Count & $0$ model calls & Mem0 F1 $+2.87$; BM25 R@10 $+1.47$
 & Constant all-ones action mask \\
\bottomrule
\end{tabular}}
\end{table}

\subsection{Index Policies, Selection, and Exposure}
\label{app:selected-policy}

\paragraph{Shared interface.} After the answer context for query $t$ is
sealed, a policy returns a binary mask over the host's native feedback list
$\mathcal{B}_t$; only masked entries increment $n_i$. The count rule is the
all-ones mask. The frozen policy prompts Qwen/Qwen3-8B with the query and
feedback memories' text at temperature zero, without task-specific fine-tuning
(\LinkFigure{fig:prompt-mask}). The trained policy uses the same student
backbone with LoRA adapters and a ten-slot head, fitted by class-balanced
binary cross-entropy to masks generated by GPT-5.6-Sol under the same prompt. Every arm
shares the host, candidate-generation interface and pool cap, suppression form of
\LinkEquation{eq:score}, and the answerer.
The student's fitting loss does not directly optimise the prospective
utility criterion in \LinkEquation{eq:prospective}. 

\paragraph{Model and fitting details.}
The teacher is GPT-5.6-Sol with fixed parameters; the student and frozen
selector use Qwen/Qwen3-8B~\citep{yang2025qwen3}. The student trains LoRA
adapters and a linear head while freezing the backbone's original weights.
For feedback slot $j$ containing memory $m_i$, the final-token representation
$h_\phi(\mathcal O_t)$ gives
\begin{equation}
p_{\theta,i}(\mathcal O_t)=\operatorname{sigmoid}
 \big(\mathbf w_j^\top h_\phi(\mathcal O_t)+b_j\big),
\qquad a_{i,t}=\mathbf 1\{p_{\theta,i}(\mathcal O_t)\ge 0.5\}.
\label{eq:policy-head}
\end{equation}
The trainable parameters are the adapters $\phi$ and head $(\mathbf W,\mathbf b)$.
The observation numbers memories in native retrieval order and clips the
query to $300$ characters and each memory to $200$ characters before applying
the chat template. The ten-slot linear head reads the final-token decoder
state; only the first $b_t$ logits are used for $0\leq b_t\leq10$.
Unused slots are excluded from the loss and actions; feedback longer than ten
is rejected rather than truncated or grouped. Overlength token sequences
are also rejected after the stated text clipping. LoRA uses rank
$8$, scale $16$, and zero dropout on the attention query and value projections.
AdamW fits the adapters and head at learning rate $10^{-5}$, batch size $1$,
zero weight decay, and gradient clipping at $1$, for two shuffled epochs.
The context limit is $4096$ tokens and the action threshold is $0.5$.
Class weights are $w_c=N_{\mathrm{fit}}/(2N_c)$, where $N_c$ counts fitting
slots with label $c$ and $N_{\mathrm{fit}}=N_0+N_1$.
The binary loss is $\ell_{\mathrm{BCE}}(p,y)=-y\log p-(1-y)\log(1-p)$;
empty feedback contributes neither loss nor updates. The fitting set contains
$3676$ keep and $5344$ suppress slots, giving
weights $1.2269$ and $0.8439$ in \LinkEquation{eq:policy-imitation}.
We use each seed's final checkpoint, with no validation-based early stopping.

\paragraph{Selection.} The frozen Qwen/Qwen3-8B policy and the
class-balanced Qwen/Qwen3-8B student were selected on a $638$-question LoCoMo
scope and then fixed across hosts and benchmarks. Student fitting used
$902$ observations from six conversations; trained results average the
three prespecified seeds. The teacher-mask objective fits the action
interface, while later-query utility evaluates the resulting access state.

\paragraph{Exposure status.}
\label{app:result-provenance}
LoCoMo supplied fitting and selection data.
Selected policies and suppression strengths were fixed before the reported
MemBench and GroupMem evaluations. These benchmarks supplied no fitting
traces; their deployment retrievals update state, and their evaluation labels
score outputs without entering runtime observations. This is transfer of the
fixed policy, not of memory-ID state between stores. Paired intervals describe variation within the evaluated scope;
LoCoMo results therefore do not constitute independent generalisation evidence.

\subsection{Policy and Teacher Prompts}
\label{app:prompts}

The frozen Qwen/Qwen3-8B policy and the GPT-5.6-Sol teacher
share one mask prompt (\LinkFigure{fig:prompt-mask}). The system message fixes
the decision criterion; the user message carries the current query and the
native feedback list, with the query clipped to $300$ characters and each
memory to $200$; a reply that does not parse is retried with a format
reminder. The teacher generates masks offline at temperature zero using
this prompt. These input--mask pairs fine-tune the Qwen/Qwen3-8B student;
the frozen comparator uses the prompt directly at temperature zero without
this fine-tuning.

\begin{figure}[!t]
\centering
\begin{minipage}{\linewidth}
\begin{promptbox}
System: You maintain the access index of an agent's long-term memory. A query was issued and the retriever returned the memories below. Decide, for each, whether it was retrieved DESERVEDLY (it bears on the query) or whether it INTRUDED (it was surfaced by lexical accident, generic phrasing or topical adjacency and does not bear on the query). Your verdict does not reorder these results. It is written into a persistent penalty that affects which memories this agent can reach for ALL FUTURE queries, so judge the memory's tendency to intrude rather than its rank here, and prefer 'noop' when unsure: suppressing a memory that merely ranked low costs future access to it. Reply with JSON only: {"verdicts": [{"id": <int>, "a": "suppress"|"noop"}, ...]} listing only the memories you judge to have intruded, plus nothing else. An empty list is a valid and often correct answer.

User: QUERY: {query}

RETRIEVED:
0. {feedback memory 0}
1. {feedback memory 1}
...

Format reminder appended after an unparseable reply: Return only JSON verdicts with unique integer IDs in {0, ..., k-1} and actions suppress/noop. Empty verdicts are valid.
\end{promptbox}
\caption{\textbf{Mask prompt} shared by frozen Qwen/Qwen3-8B and the
GPT-5.6-Sol teacher that supplies student training targets. Runtime verdicts update the persistent counts in \LinkEquation{eq:masked-update}; they
do not reorder the current list.}
\label{fig:prompt-mask}
\end{minipage}
\end{figure}

\subsection{Prospective Value of a Policy Update}
\label{app:prospective-value}
\label{app:prospective}

\paragraph{Prospective criterion.}
An update made at step $t$ produces $\mathcal{I}_{t+1}$, and its value is what
$\mathcal{I}_{t+1}$ does for the queries that follow, relative to leaving the
index alone. For a fixed index, write
$V(q;\mathcal{I})=\mathcal{U}(q,\mathcal{A}(\cdot\mid q,
\mathcal{R}(q;\mathcal{C},\mathcal{I})))$ for downstream task utility.
\begin{equation}
\mathrm{Val}(\Delta\mathcal{I}_t) \;=\; \mathbb{E}_t\Big[\textstyle\sum_{\tau=t+1}^{t+H}
\big(V(q_\tau;\mathcal{I}_{t+1})-V(q_\tau;\mathcal{I}_{t})\big)\Big],
\qquad
\mathcal{I}_{t+1}=F(\mathcal{I}_t,\mathcal{O}_t),
\label{eq:prospective}
\end{equation}
Here the expectation is over future queries conditional on information
available at step $t$. Both branches keep their index states fixed,
so this criterion isolates the prospective value of one update; it is not a
full trajectory objective with subsequent policy actions. This criterion
defines what makes an update useful; it supplies no future information to
the runtime policy. Teacher-mask imitation supplies a tractable learning
objective, but does not directly optimise this prospective criterion.

\subsection{Store, Retriever, and Score Normalisation}
\label{app:store}

AMIE does not modify stored content; its policy can read the retrieved
memory text. Host construction uses each system's native content procedure.

The adapters in \LinkAppendix{app:attachment} first map native scores to
nonnegative affinities (for example, inverse squared distance for Mem0 and
shifted cosine for A-MEM), then apply positive query-wise scaling. The
mapping preserves native preference; it is not itself necessarily a
proportional normalisation. The selection identity in
\LinkEquation{eq:wrapper} concerns IDs and tie order under the same pool.
Granularity interventions group turns within the BM25 replay
(\LinkAppendix{app:granularity-ladder}).

\paragraph{State-maintenance complexity.}
After sealing the answer context, the wrapper updates the state of the
native feedback memories, $\mathcal{A}_t=\{i:m_i\in\mathcal{B}_t\}$,
with $b_t=|\mathcal{B}_t|$. Thus
\begin{equation}
T_{\mathrm{update}}=O(b_t),
\qquad
M_{\mathrm{state}}=O(|\mathcal{C}|)\ \text{dense, or }
O\Big(\big|\textstyle\bigcup_{\tau<t}\mathcal{B}_\tau\big|\Big)\ \text{sparse}.
\label{eq:complexity}
\end{equation}
These bounds cover count-state maintenance. The bound on $u_i$ is per-memory
suppression magnitude, not a bound on accumulated counts or total storage.
Candidate acquisition and re-ranking have separate costs, and trained/frozen
policies add inference cost.

\subsection{Retrieval traffic as an interference signal}
\label{app:traffic-risk}
\paragraph{Why retrieval traffic grounds the policy.}
The native retrieval event identifies memories that compete for limited
slots. Its empirical frequency is
\begin{equation}
\rho_{i,t}=\frac{1}{t-1}\sum_{\tau<t}\mathbf{1}[m_i\in\mathcal{B}_\tau],
\qquad t>1.
\label{eq:occupancy}
\end{equation}
This measures exposure, not harm: retrieving a frequently useful memory
need not be an intrusion. Let $Z_i(q)=1$ indicate that a retrieved memory
does not answer $q$. Under a deployment query distribution, interference
risk decomposes as
\begin{equation}
\begin{aligned}
\mathrm{Risk}(m_i)
&:=\Pr[m_i\in\mathcal{R}(q)\ \wedge\ Z_i(q)=1]\\
&=\underbrace{\Pr[m_i\in\mathcal{R}(q)]}_{\text{exposure}}
\,\underbrace{\Pr[Z_i(q)=1\mid m_i\in\mathcal{R}(q)]}_{\text{conditional interference}}.
\end{aligned}
\label{eq:risk}
\end{equation}
The count special case uses exposure alone, with $n_{i,t}=(t-1)\rho_{i,t}$.
Frozen and trained policies condition on observed retrieval context for
selective suppression. Their masks remain
proxies for harmful competition; neither imitation nor exposure alone
guarantees improved utility.

\section{Detailed Results}
\label{app:full-results}

\LinkTable{tab:retrieval-main} and~\LinkTable{tab:paired-main} are the sole source
for host-comparison scores. \LinkTable{tab:detail-locomo} adds uncertainty only
for selected pairs whose stored predictions reproduce those displayed scores;
the matched replay studies address the mechanism of persistent access.

\begin{table}[t]
\caption{\textbf{Reproducible LoCoMo utility pairs.} Scores and displayed
$\Delta$ follow \LinkTable{tab:paired-main}; intervals use the unrounded
per-question differences and conversation-cluster $95\%$ bootstrap.}
\label{tab:detail-locomo}
\centering\small
\begin{tabular}{llrrrr}
\toprule
Host & Policy & w/o AMIE & +AMIE & $\Delta$ & 95\% CI \\
\midrule
\multirow{3}{*}{Mem0} & trained & 9.98 & 12.44 & +2.46 & $[+1.60,+3.24]$ \\
 & frozen & 9.98 & 12.69 & +2.71 & $[+1.73,+3.53]$ \\
 & count & 9.98 & 12.85 & +2.87 & $[+1.95,+3.80]$ \\
\midrule
\multirow{3}{*}{MemoryOS} & trained & 20.27 & 20.97 & +0.70 & $[+0.29,+1.14]$ \\
 & frozen & 20.27 & 20.83 & +0.56 & $[+0.10,+1.01]$ \\
 & count & 20.27 & 20.77 & +0.50 & $[-0.13,+1.05]$ \\
\bottomrule
\end{tabular}
\end{table}

\section{Mechanism Controls}
\label{app:count}

\paragraph{Control-figure protocols.}
\LinkFigure{tab:mechanism-controls} uses the $1{,}527$-question BM25 replay;
per-question effects use conversation-cluster $95\%$ intervals, while the
fifty-order comparisons use paired $t$ intervals over permutations.
Missing per-question records are shown as points without intervals.
\LinkFigure{tab:gold-unseen} has BM25 bucket counts $542/587/354/44$ and
MemoryOS counts $541/470/416/100$; exposure counts earlier appearances of
the question's gold among native candidates. Both protocols retain the
same question order within a comparison.
The recurrence intervention injects counts on $1{,}504$ adjacent question
pairs within this replay. These label-informed interventions do not
supply oracle information to the online policy.

\paragraph{Control definitions.}
Hindsight $\eta$ (h-$\eta$) sets each memory's state to the fraction of its
native-feedback appearances that are non-gold, using the full conversation.
Hindsight $\rho\eta$ counts those non-gold appearances; online variants use
only preceding observations. Hindsight count uses full-stream exposure
counts, while permuted state shuffles memory-specific weights.
Traffic promotion (tr-p) reverses the sign of count-based score changes;
label suppression (lb-s) and label promotion (lb-p) instead target entries
labelled as gold in preceding feedback. The direction controls scale their
strength to match the count rule's mean absolute score change on returned
items, subject to the recorded strength cap. They are diagnostic
interventions, not deployable policies. Windows retain only the preceding
$W$ feedback events; decay discounts counts with half-life $h$. Their gains
are relative to the unmodified host. Reset effects compare clearing the
state at the stated stream fractions with uninterrupted accumulation under
the same question order.

\begin{table}[t]
\centering\footnotesize
\setlength{\tabcolsep}{3pt}
\begin{minipage}[t]{0.47\linewidth}
\centering
\caption{\textbf{Fresh and reinforced evidence in the LoCoMo replay.}
A question is reinforced when its gold overlaps an earlier question's gold.
Values are recall-point differences from the same host.}
\label{tab:fresh-reinforced}
\begin{tabular}{llrr}
\toprule
\textbf{Host} & \textbf{Arm} & \textbf{fresh} & \textbf{reinforced} \\
 & & ($n{=}788$) & ($n{=}739$) \\
\midrule
\multirow{3}{*}{BM25} & count & $\mathbf{+2.06}$ & $+0.86$ \\
 & frozen & $+1.14$ & $\mathbf{+1.41}$ \\
 & trained & $+0.86$ & $\mathbf{+1.31}$ \\
\midrule
\multirow{3}{*}{MemoryOS} & count & $+0.76$ & $+0.44$ \\
 & frozen & $\mathbf{+0.97}$ & $+0.85$ \\
 & trained & $-0.00$ & $+0.37$ \\
\bottomrule
\end{tabular}

{\footnotesize Bold: conversation-cluster $95\%$ interval excludes zero
(count fresh $[+0.89,+3.23]$; frozen reinforced $[+0.21,+2.55]$; trained
reinforced $[+0.42,+2.10]$; frozen fresh on MemoryOS $[+0.12,+1.80]$).\par}
\end{minipage}\hfill
\begin{minipage}[t]{0.50\linewidth}
\centering
\caption{\textbf{Count effects across retrieval depths on BM25 and MemoryOS.}
Depth $k$ counts turns for BM25 and pages for MemoryOS; recall scores source
turns. Intervals are conversation-cluster $95\%$: the first BM25 interval
including zero is at $k=50$, while all MemoryOS intervals include zero.}
\label{tab:depth}
\resizebox{\linewidth}{!}{%
\begin{tabular}{lrr}
\toprule
$k$ & \textbf{BM25} & \textbf{MemoryOS} \\
\midrule
$3$ & $+2.1$ $[+0.8,+3.2]$ & $+1.2$ $[-0.4,+2.6]$ \\
$5$ & $+2.7$ $[+1.6,+3.9]$ & $+0.6$ $[-1.1,+2.2]$ \\
$7$ & - & $+0.6$ $[-0.5,+1.9]$ \\
$10$ & $+1.5$ $[+0.7,+2.2]$ & capped at $7$ \\
$20$ & $+1.7$ $[+0.7,+2.8]$ & - \\
$50$ & $+0.4$ $[-0.0,+0.9]$ & - \\
\midrule
nDCG@10 & $0.384\to0.409$ & - \\
 & $+2.5$ $[+1.7,+3.2]$ & \\
\bottomrule
\end{tabular}}
\end{minipage}
\end{table}

\begin{figure}[t]
\centering
\includegraphics[width=\linewidth]{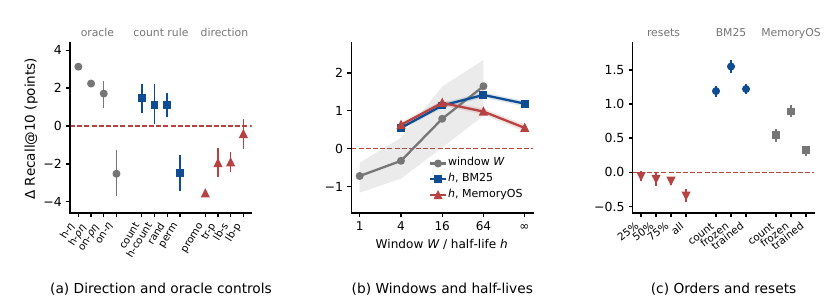}
\caption{\textbf{Controlled state interventions on BM25 and MemoryOS.}
(a) BM25 oracle and direction controls; (b) windows and decay; (c) resets
and order comparisons. Recall-point differences use conversation-cluster
$95\%$ intervals for per-question comparisons and paired $t$ intervals
for fifty-order comparisons; missing intervals are omitted.
\LinkAppendix{app:count} defines the controls and their baselines.}
\label{tab:mechanism-controls}
\end{figure}

\begin{figure}[t]
\centering
\includegraphics[width=\linewidth]{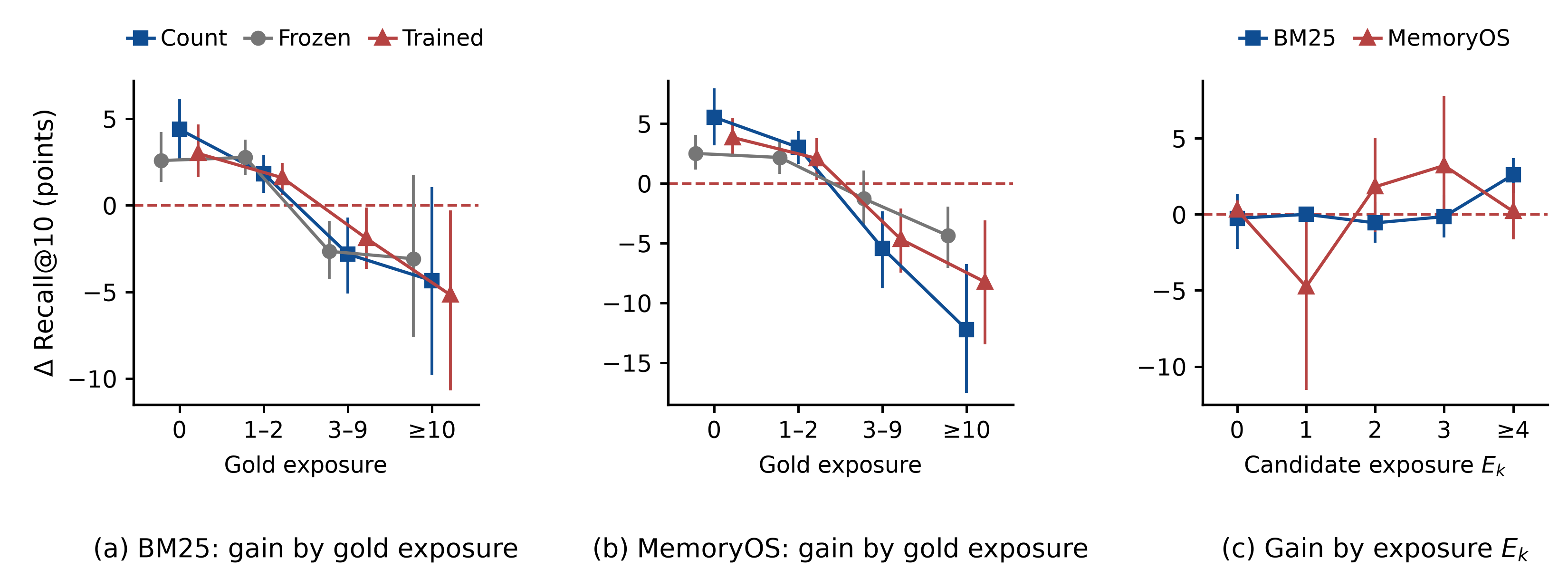}
\caption{\textbf{Policy effects depend on evidence exposure in the LoCoMo replay.}
Panels (a,b) stratify BM25 and MemoryOS by prior gold exposure; (c) stratifies
count effects by candidate exposure.
Whiskers show conversation-cluster $95\%$ intervals for recall-point differences.}
\label{tab:gold-unseen}
\label{tab:ek-buckets}
\end{figure}

\paragraph{Where the two hosts differ} (\LinkFigure{tab:gold-unseen}). The
frozen mask trades rather than dominates: it gives up part of the never-seen gain (paired count minus frozen
$+1.8$ $[+0.6,+3.1]$ on BM25 and $+3.0$ $[+1.5,+4.8]$ on MemoryOS) and
recovers it at high exposure ($-4.1$ $[-7.3,-1.0]$ and $-7.9$ $[-11.8,-2.8]$
on the two highest MemoryOS buckets), netting $+0.2$ for counting on BM25 and
$+0.3$ for the mask on MemoryOS. A conditional-harm signal moves the cost of
suppression; it does not remove it.

\paragraph{Depth, ranking quality, and reachable headroom}
(\LinkTable{tab:depth}). The same replay
scored at depths $k\in\{3,5,10,20,50\}$ gives $+2.1$ $[+0.8,+3.2]$, $+2.7$
$[+1.6,+3.9]$, $+1.5$ $[+0.7,+2.2]$, $+1.7$ $[+0.7,+2.8]$ and $+0.4$
$[-0.0,+0.9]$: a head-of-list intervention whose interval first touches zero
as the returned context includes more of the candidate pool. Ranking quality improves
throughout (nDCG@10 $0.384\to0.409$, $+2.5$ $[+1.7,+3.2]$). Reachable
headroom (gold outside the top ten but inside the top hundred) is $27.5\%$
on BM25 and $22.8\%$ on MemoryOS and correlates with the per-conversation gain
(Spearman $0.41$ and $0.32$) while being nearly independent of repeat
occupancy (Spearman $-0.18$ and $-0.06$).

\paragraph{Repeat occupancy tracks different trade-offs within each host.}
The share of BM25's top ten already retrieved by earlier questions spans
$0.67$--$0.79$ across conversations and correlates with per-conversation gain
at Pearson $+0.70$ (leave-one-out $0.61$--$0.79$). The MemoryOS correlation
is negative and unstable. Occupancy therefore describes retrieved traffic;
evidence headroom is needed to interpret the opportunities it creates.

\paragraph{Controlled recurrence.} Intervening
directly on $1{,}504$ adjacent question pairs isolates the trade the state
makes. Suppressing a question's
competitors helps that question ($+2.00$ to $+5.29$ points as intensity
grows) and costs the next one ($-0.19$ to $-1.03$), while suppressing its
useful memories is far more damaging ($-6.27$ to $-29.85$). This asymmetry is consistent with the benefit of decay relative to
unbounded accumulation and of persistence relative to per-query resets (\LinkFigure{tab:mechanism-controls}).

\section{Replay Studies of the Index State}
\label{app:replay-studies}

The studies in this section share one matched replay: the main-table BM25
protocol on the $1{,}527$ evidence-valid LoCoMo questions, released order,
$\kappa{=}3$, $\lambda_p{=}0.5$, a $100$-candidate pool, and the all-ones
mask. Every script reproduces the published arms ($51.94\to53.41$) before it
scores a variant, every plan is sealed before scoring, and every interval is
a conversation-cluster bootstrap ($10{,}000$ draws). No model is called.

\begin{figure}[t]
\centering
\includegraphics[width=\linewidth]{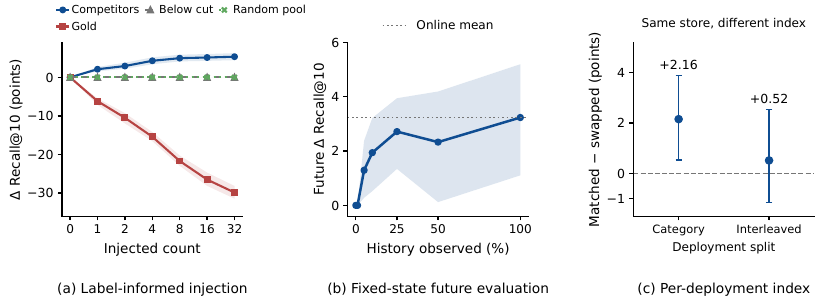}
\begin{subcaptiongroup}
\phantomsubcaption\label{fig:injection}
\phantomsubcaption\label{fig:sufficiency}
\phantomsubcaption\label{fig:index-swap}
\end{subcaptiongroup}
\caption{\textbf{Persistent state interventions in the BM25 replay.}
Panels show (a) label-informed count injection on $1{,}527$ questions,
(b) frozen-state evaluation on $386$ future questions, and (c) matched-minus-swapped
state on $381$ category-split and $387$ interleaved future questions.
Bands and whiskers are conversation-cluster $95\%$ intervals on recall-point
differences; injection is a diagnostic intervention, not a runtime policy.}
\label{fig:replay-one}
\end{figure}

\subsection{Access history alone changes what is retrievable}
\label{app:injection}

The query, host scores, and gold labels are matched, and only the
access history of one target set is changed: the operator starts from an
empty state and the count of every entry in the target set is set to $n$
before the question is scored. \LinkTable{tab:injection} and \LinkFigure{fig:injection} report the
four target sets. Setting the competitors' history alone raises evidence
Recall@10 monotonically from $51.94$ to $57.29$, with $118$ questions
improving and none worsening at $n{=}32$; setting the same history on the
gold entries costs up to $29.9$ points; setting it on entries that already sit
below the cut, or on random pool entries of the same score mass, changes
nothing at any level. The same content therefore admits a $35$-point range of
recall depending only on which memories the past retrieved, which is the
separation in \LinkSection{sec:formulation}: identical content supports different
retrieval outcomes under different access states.

The natural online state sits between the two curves. Under the released
order the competitors of a question carry a median count of $2$ ($36.5\%$
reach $4$), but so do the gold entries themselves (median $1$, $16\%$ reach
$4$), and the online gain of $+1.48$ is the net of the two: the rule earns
where it demotes competitors and pays where the traffic has also counted
the evidence, consistent with the exposure-stratified results of \LinkSection{sec:exp-why}.

\begin{table}[htbp]
\caption{\textbf{Causal traffic injection} ($\Delta$Recall@10 against the
same question with an empty state, conversation-cluster $95\%$ intervals,
$n{=}1{,}527$). Only the count of the target set is changed; query,
host scores, gold and store are identical across columns. Under the online
rule the competitors' natural median count is $2$ and the gold entries' is
$1$.}
\label{tab:injection}
\centering\small
\setlength{\tabcolsep}{4pt}
\begin{tabular}{rllll}
\toprule
$n$ & \textbf{competitors (non-gold top ten)} & \textbf{gold entries of the top ten} & \textbf{ranks 11--20} & \textbf{random pool} \\
\midrule
$1$  & $+2.10$ $[+1.57,+2.68]$ & $-6.17$ $[-7.12,-5.43]$ & $0.00$ & $0.00$ \\
$2$  & $+2.93$ $[+2.20,+3.80]$ & $-10.49$ $[-11.94,-9.28]$ & $0.00$ & $0.00$ \\
$4$  & $+4.29$ $[+3.40,+5.19]$ & $-15.40$ $[-16.53,-14.19]$ & $0.00$ & $0.00$ \\
$8$  & $+4.97$ $[+4.13,+5.87]$ & $-21.69$ $[-22.94,-20.26]$ & $0.00$ & $0.00$ \\
$16$ & $+5.14$ $[+4.28,+6.05]$ & $-26.51$ $[-28.13,-24.69]$ & $0.00$ & $0.00$ \\
$32$ & $+5.35$ $[+4.45,+6.28]$ & $-29.86$ $[-31.45,-28.16]$ & $0.00$ & $0.00$ \\
\midrule
\multicolumn{5}{l}{online count rule on the same questions: $+1.48$ $[+0.70,+2.21]$} \\
\bottomrule
\end{tabular}
\end{table}

\subsection{How much traffic the state needs}
\label{app:sufficiency}

Each conversation's stream is split into an accumulation segment (the first
$75\%$ of questions) and a fixed future set (the last $25\%$, $386$
evidence-valid questions, off recall $65.67$). The state accumulates over
the first fraction $f$ of the accumulation segment, is frozen, and every arm
is scored on the identical future set (\LinkTable{tab:sufficiency},
\LinkFigure{fig:sufficiency}). Ten
observed questions yield a zero point estimate; $59$ in total give $+1.30$ $[+0.28,+2.38]$,
$288$ give $+2.72$ $[+1.34,+3.95]$, and the full segment gives $+3.24$
$[+1.11,+5.20]$, the same value the fully online arm reaches on that future
set. The $59$ observations are pooled across ten conversations ($3$--$7$ per
conversation), not $59$ per conversation. The point estimates suggest early
benefits but are not monotone in observed traffic.

The online arm and the frozen full-history arm agree to three decimals
although their top-ten lists differ on $306$ of the $386$ future questions:
the traffic that arrives inside the future set reorders entries below the
gold but changes gold recall on only four questions, and those changes
cancel. On this stream the state built before the future set already
contains the competitors that future set contests, which is what the
persistence results of \LinkSection{sec:exp-why} would predict.

\begin{table}[htbp]
\caption{\textbf{Traffic sufficiency} on the fixed future set ($386$
questions; $\Delta$Recall@10 against the unmodified host, conversation-cluster
$95\%$ intervals). The state is frozen after observing the first fraction
$f$ of each conversation's accumulation segment; the online row keeps
updating inside the future set.}
\label{tab:sufficiency}
\centering\small
\begin{tabular}{rrrl}
\toprule
$f$ & \textbf{questions observed (total)} & \textbf{Recall@10} & $\Delta$ \textbf{[95\% CI]} \\
\midrule
$0$ & $0$ & $65.67$ & $0.00$ \\
$0.01$ & $10$ & $65.67$ & $0.00$ \\
$0.05$ & $59$ & $66.97$ & $+1.30$ $[+0.28,+2.38]$ \\
$0.10$ & $114$ & $67.62$ & $+1.94$ $[+0.53,+3.22]$ \\
$0.25$ & $288$ & $68.39$ & $+2.72$ $[+1.34,+3.95]$ \\
$0.50$ & $576$ & $68.01$ & $+2.33$ $[+0.12,+4.20]$ \\
$1.00$ & $1{,}154$ & $68.91$ & $+3.24$ $[+1.11,+5.20]$ \\
\midrule
online & all & $68.91$ & $+3.24$ $[+1.01,+5.33]$ \\
\bottomrule
\end{tabular}
\end{table}

\subsection{Policy selectivity across competitor and gold exposure}
\label{app:exposure-map}

The saved per-question deltas of the three replay policies are binned by two
exposure indices recomputed in released order: $E_k$, the number of native
top-ten candidates previously retrieved, and the gold exposure of
\LinkFigure{tab:gold-unseen}. \LinkTable{tab:exposure-map} and \LinkFigure{fig:exposure-map} report
the high-competition row, which holds $1{,}389/1{,}527$ ($91.0\%$) of the questions, and the pooled
low-competition cells. Where competition has been observed and the evidence
has not, counting beats the semantic mask by $+2.22$ $[+0.79,+3.90]$; as
soon as the evidence itself has been retrieved the sign turns, with the mask
ahead by $0.89$ at one or two exposures (interval touching zero) and less
harmful at ten or more. Where competition has not been observed, the intervals for all three arms include zero; this does not establish
equivalence to the host or absence of an effect.

These label-assisted analyses diagnose when suppression helps or harms.
Counting gives the larger point estimate on unexposed evidence, whereas
selectivity can reduce damage to evidence that earlier traffic retrieved.
Turning the diagnosis into an online decision rule requires observable
proxies for gold exposure and reachable headroom; neither quantity is
available from unlabelled retrieval traffic alone.

\begin{table}[htbp]
\caption{\textbf{Count against the semantic arms across competitor and gold
exposure} (BM25 replay, $\Delta$Recall@10 against the host, conversation-cluster
$95\%$ intervals on the paired count-minus-frozen difference). $E_k$ is the
number of the host's ten candidates earlier traffic had already retrieved;
gold exposure is the number of earlier appearances of the question's own
gold turns.}
\label{tab:exposure-map}
\centering\small
\setlength{\tabcolsep}{4pt}
\begin{tabular}{llrrrrl}
\toprule
$E_k$ & \textbf{gold exposure} & $n$ & \textbf{count} & \textbf{frozen} & \textbf{trained} & \textbf{count $-$ frozen [95\% CI]} \\
\midrule
$\geq 4$ & $0$ & $439$ & $+5.33$ & $+3.11$ & $+3.65$ & $+2.22$ $[+0.79,+3.90]$ \\
$\geq 4$ & $1$--$2$ & $559$ & $+2.00$ & $+2.90$ & $+1.71$ & $-0.89$ $[-1.82,+0.00]$ \\
$\geq 4$ & $3$--$9$ & $345$ & $-2.89$ & $-2.73$ & $-1.96$ & $-0.16$ $[-1.67,+1.05]$ \\
$\geq 4$ & $\geq 10$ & $46$ & $-4.16$ & $-2.96$ & $-4.96$ & $-1.20$ $[-4.35,+0.27]$ \\
\midrule
$\leq 3$ & any & $138$ & $-0.07$ & $+0.24$ & $+0.05$ & intervals include zero \\
\bottomrule
\end{tabular}
\end{table}

\subsection{An evolvability map: competition and headroom}
\label{app:evolvability}

Repeat occupancy and gold headroom describe different aspects of the BM25
replay (\LinkTable{tab:evolvability}). The largest count point estimate,
$+15.99$ $[+12.73,+19.00]$, occurs on the $230$ high-occupancy questions with
missing gold inside the pool. Corresponding mid- and low-occupancy estimates
are $+3.36$ and $+0.61$; questions with all gold already in the top ten tend
to lose. This stratification uses labels and does not establish necessary
conditions for every policy or host.

A separate intervention starts with an empty state and increments only
non-gold entries in the original top ten. It rescues gold at ranks
$11$--$20$ ($+24.7$ at one increment and $+62.9$ at $32$, $n=130$) but not
at deeper ranks. Because at most ten entries are intervened on, a target can
pass at most those entries. The factor-of-two score bound at $\lambda_p=0.5$
does not imply a rank-distance bound for AMIE's full state space: states that
penalise other pool entries were not searched by this intervention.

\begin{table}[htbp]
\caption{\textbf{Index evolvability map} (online count rule, BM25 replay,
$\Delta$Recall@10 against the host, conversation-cluster $95\%$ intervals).
Rows are repeat-occupancy bins of the host's top ten ($<0.4$, $0.4$--$0.8$,
$\geq0.8$); columns class the question's missing gold as already in the top
ten, inside the pool at ranks $11$--$100$, or beyond the pool.}
\label{tab:evolvability}
\centering\small
\setlength{\tabcolsep}{4pt}
\resizebox{\linewidth}{!}{%
\begin{tabular}{lrrr}
\toprule
\textbf{Repeat occupancy} & \textbf{no gold missing} & \textbf{gold reachable (11--100)} & \textbf{gold beyond the pool} \\
\midrule
low ($n{=}138$) & $0.00$ ($n{=}34$) & $+0.61$ $[+0.00,+1.82]$ ($n{=}55$) & $-1.02$ $[-3.26,+0.00]$ ($n{=}49$) \\
mid ($n{=}454$) & $-1.93$ $[-3.78,+0.00]$ ($n{=}207$) & $+3.36$ $[+0.79,+6.09]$ ($n{=}135$) & $-0.74$ $[-1.53,+0.00]$ ($n{=}112$) \\
high ($n{=}935$) & $-2.78$ $[-3.59,-1.86]$ ($n{=}485$) & $+15.99$ $[+12.73,+19.00]$ ($n{=}230$) & $-0.11$ $[-0.38,+0.00]$ ($n{=}220$) \\
\bottomrule
\end{tabular}}

\end{table}

\subsection{How much of the reachable headroom the rule captures}
\label{app:ceilings}

\LinkTable{tab:ceilings} compares the online rule with three oracle
references, not an exhaustive optimisation over index states. Suppressing
only non-gold original top-ten entries with saturated counts gives $+5.54$
$[+4.61,+6.46]$ points; the online gain is $27\%$ of this particular
intervention. Ideal selection of gold from the fixed $100$-candidate pool
gives $+20.43$ $[+18.44,+22.61]$, while perfect recall gives $+48.06$.
The latter two are recall upper bounds for their stated candidate access,
not guarantees attainable by AMIE or improvements in generation utility.

The oracle-count intervention does not optimise every feasible state or
identify a functional-form ceiling. Its gap from online counting combines
oracle label access with different choices of memories to suppress; it is
not an isolated estimate of the cost of unconditional counting. The saved
hindsight references use future labels and are diagnostic only.

\begin{table}[htbp]
\caption{\textbf{Oracle references for index-side recall gains.}
The top-ten intervention suppresses only original non-gold top-ten entries;
pool and perfect-recall references have broader oracle access.
Values are BM25 replay differences with conversation-cluster $95\%$ intervals,
not an optimised functional-form bound.}
\label{tab:ceilings}
\centering\small
\begin{tabular}{lrl}
\toprule
\textbf{Arm} & $\Delta$ \textbf{[95\% CI]} & \textbf{online gain / reference gain} \\
\midrule
online count rule & $+1.48$ $[+0.70,+2.21]$ & --- \\
hindsight $\rho\eta$ oracle (saved) & $+2.24$ & $66\%$ \\
hindsight $\eta$-only oracle (saved) & $+3.13$ & $47\%$ \\
oracle top-ten suppression & $+5.54$ $[+4.61,+6.46]$ & $27\%$ \\
ideal pool recall & $+20.43$ $[+18.44,+22.61]$ & $7\%$ \\
perfect-recall reference & $+48.06$ $[+46.34,+49.81]$ & $3\%$ \\
\bottomrule
\end{tabular}
\end{table}

\begin{figure}[t]
\centering
\includegraphics[width=\linewidth]{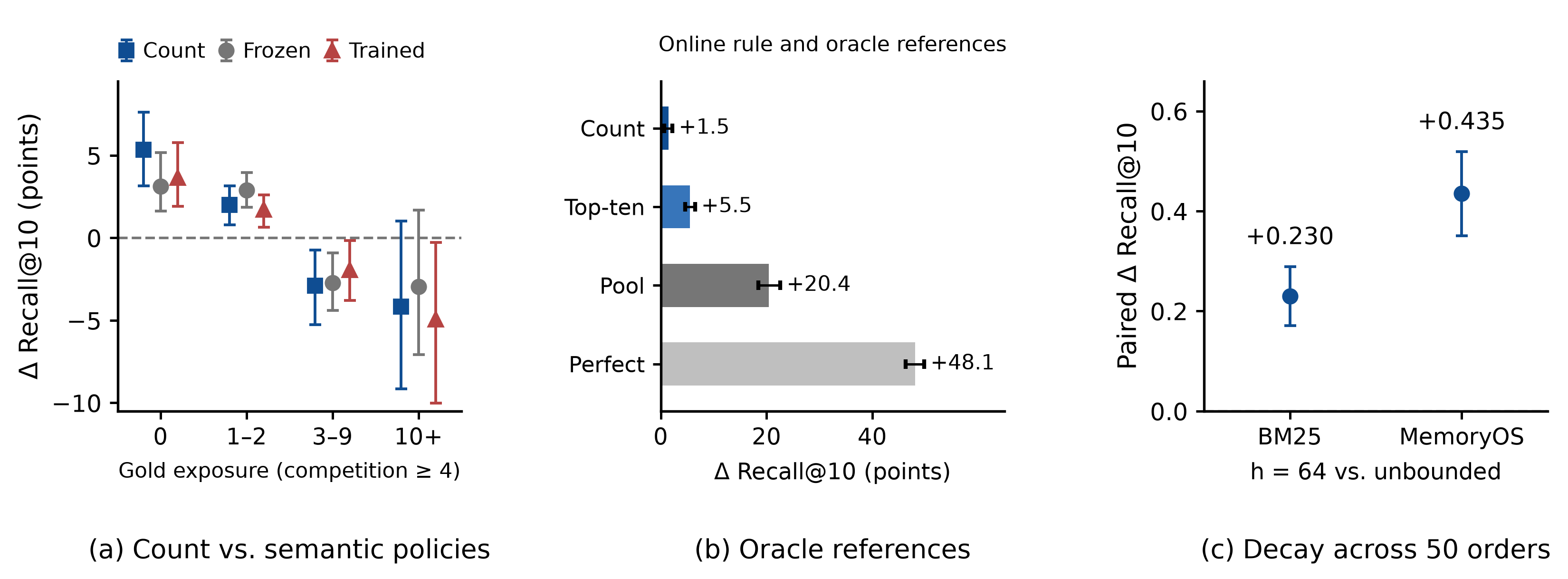}
\begin{subcaptiongroup}
\phantomsubcaption\label{fig:exposure-map}
\phantomsubcaption\label{fig:ceilings}
\phantomsubcaption\label{fig:decay-orders}
\end{subcaptiongroup}
\caption{\textbf{Selectivity, oracle references, and decay.}
Panels (a,b) show BM25 replay recall differences with conversation-cluster
$95\%$ intervals; (c) compares decay to unbounded counting across fifty orders
with $95\%$ $t$ intervals.
Oracle top-ten suppression is a restricted intervention, not a state-space optimum.}
\label{fig:replay-two}
\end{figure}

\subsection{Deployment-specific index states over the same store}
\label{app:swap}

Two deployments share a store when two question streams run over the same
conversation. The LoCoMo categories are different information needs over
one store, so each category's questions form a deployment; as a control, the
even and odd questions of the released order form two deployments with the
same distribution. For every deployment the state accumulates over the
first $75\%$ of its stream and is frozen, and its last $25\%$ is scored under
its own state, under every other deployment's state, and under the union of
the others (\LinkFigure{fig:index-swap}; \LinkTable{tab:swap}). Across the category deployments the
matched state beats the swapped state by $+2.16$ $[+0.54,+3.90]$ on $381$
paired future questions, while across the interleaved control the same
comparison is $+0.52$ $[-1.16,+2.54]$. The category result supports
useful deployment-specific information in the index; the interleaved result
is inconclusive, rather than evidence of equivalence. Neither comparison
optimises the index state or proves that the matched state is optimal.

The control also answers the transfer question for the evolved state. Within
one store the state transfers to new questions of the same deployment (the
matched arms are positive on three of the four categories that have enough
questions), while the transfer cost between interleaved streams is inconclusive; a per-memory state cannot transfer to a new store at all, because its
entries name memories that do not exist there, so cross-store transfer is a
property of policies rather than of states, and the policy result of
\LinkSection{sec:exp-count-default} (students fitted on one host, deployed unchanged
to the others) is the answer to that question.

\begin{table}[htbp]
\caption{\textbf{Content--index swap} (BM25 replay, frozen states,
$\Delta$Recall@10 against the unmodified host on each deployment's future
set, conversation-cluster $95\%$ intervals). The paired rows pool every
future question of a split and compare its own state with a state evolved on
another deployment of the same store.}
\label{tab:swap}
\centering\small
\setlength{\tabcolsep}{4pt}
\resizebox{\linewidth}{!}{%
\begin{tabular}{llrl}
\toprule
\textbf{Split} & \textbf{Arm} & $n$ & $\Delta$ \textbf{[95\% CI]} \\
\midrule
\multirow{5}{*}{categories (different deployments)}
 & category 1, own state & $71$ & $+1.76$ $[+0.01,+3.65]$ \\
 & category 2, own state & $77$ & $+5.20$ $[+0.00,+11.77]$ \\
 & category 4, own state & $213$ & $+3.05$ $[+0.50,+5.48]$ \\
 & \textbf{paired: own $-$ other deployment's state} & $381$ & $\mathbf{+2.16}$ $[+0.54,+3.90]$ \\
 & paired: own $-$ union of the others & $381$ & $+1.12$ $[-1.16,+4.10]$ \\
\midrule
\multirow{3}{*}{even / odd (same distribution)}
 & even, own state & $194$ & $+2.58$ $[+0.00,+5.20]$ \\
 & odd, own state & $193$ & $+3.11$ $[-0.50,+6.92]$ \\
 & paired: own $-$ other deployment's state & $387$ & $+0.52$ $[-1.16,+2.54]$ \\
\bottomrule
\end{tabular}}

\end{table}

\subsection{Memory granularity sets index evolvability}
\label{app:granularity-ladder}

The same conversations are re-indexed at unit sizes of one, two, four and
eight consecutive turns of a session, and at whole sessions; a unit is gold
for a question when it contains any gold turn, and the online count rule is
scored against the static index at each rung with the fixed replay settings,
with the retrieval depth fixed (\LinkTable{tab:granularity-ladder}). Evidence
purity, the share of a retrieved gold unit's turns that are gold for the
question, falls from $1.00$ at turn level to $0.51$, $0.27$, $0.14$ and
$0.05$, and the count gain falls with it, from $+1.48$ $[+0.70,+2.21]$ to
$+0.48$, $+0.09$, $+0.07$ and $+0.04$, with the numbers of questions improved
and worsened becoming similar from two turns per unit on. Static recall rises
across the ladder because the unit-level task becomes easier, while the measured count advantage approaches zero. When a unit mixes
evidence with its surroundings, demoting the unit demotes the evidence, and suppression has no
clear separation between competitor and evidence. This result identifies a
risk in the evaluated count replay, not a requirement that every host use the
finest unit or a guarantee that coarser units cannot benefit.

\begin{table}[htbp]
\caption{\textbf{Granularity ladder} (BM25 replay, online count rule,
Recall@10 over gold units, conversation-cluster $95\%$ intervals; recall
levels are not comparable across rungs). Purity is the mean share of gold
turns in a retrieved gold unit.}
\label{tab:granularity-ladder}
\centering\small
\setlength{\tabcolsep}{4pt}
\begin{tabular}{lrrrrlrr}
\toprule
\textbf{Unit} & \textbf{units} & \textbf{off} & \textbf{+count} & \textbf{purity} & $\Delta$ \textbf{[95\% CI]} & \textbf{improved} & \textbf{worsened} \\
\midrule
1 turn & $5{,}882$ & $51.94$ & $53.41$ & $1.00$ & $+1.48$ $[+0.70,+2.21]$ & $57$ & $30$ \\
2 turns & $3{,}011$ & $64.22$ & $64.70$ & $0.51$ & $+0.48$ $[-0.11,+1.05]$ & $47$ & $45$ \\
4 turns & $1{,}571$ & $74.23$ & $74.32$ & $0.27$ & $+0.09$ $[-0.31,+0.54]$ & $41$ & $40$ \\
8 turns & $848$ & $81.23$ & $81.31$ & $0.14$ & $+0.07$ $[-0.58,+0.62]$ & $31$ & $31$ \\
session & $272$ & $88.91$ & $88.95$ & $0.05$ & $+0.04$ $[-0.28,+0.29]$ & $19$ & $15$ \\
\bottomrule
\end{tabular}
\end{table}

\section{Policy Ablations}
\label{app:policy-ablations}
\label{app:lambda}

\paragraph{Observation regime and suppression rate.}
The nine student variants use the same BM25 retrieval replay as
\LinkTable{tab:retrieval-main}; each regime changes the information supplied
to the action policy. \LinkTable{tab:student-grid} reports all cells and
preserves their paired intervals. Suppression rate correlates with gain
($r=0.79$), while average differences between observation regimes are
smaller than the within-regime seed variation. A random mask matched to the
frozen policy's suppression rate also recovers much of counting's gain in
this replay, supporting accumulation as a shared source of access changes.

\begin{table}[htbp]
\caption{\textbf{Observation ablation in the matched BM25 replay.}
Values are recall-point differences with conversation-cluster $95\%$ intervals;
rate is the fraction of feedback entries suppressed.}
\label{tab:student-grid}
\centering\small
\begin{tabular}{lrrl}
\toprule
Observation & Seed & Rate (\%) & $\Delta$ Recall [95\% CI] \\
\midrule
\multirow{3}{*}{traffic stats} & 1 & 52.4 & $+1.05$ $[+0.40,+1.64]$ \\
 & 2 & 58.6 & $+1.08$ $[+0.17,+1.88]$ \\
 & 3 & 71.6 & $+1.58$ $[+0.88,+2.22]$ \\
\multirow{3}{*}{stats + text} & 1 & 46.7 & $+1.13$ $[+0.49,+1.69]$ \\
 & 2 & 49.3 & $+1.33$ $[+0.77,+1.92]$ \\
 & 3 & 63.3 & $+1.50$ $[+0.79,+2.23]$ \\
\multirow{3}{*}{text only} & 1 & 41.8 & $+1.11$ $[+0.28,+1.92]$ \\
 & 2 & 49.7 & $+1.13$ $[+0.28,+2.15]$ \\
 & 3 & 67.2 & $+1.66$ $[+0.94,+2.34]$ \\
\bottomrule
\end{tabular}
\end{table}

\paragraph{Free parameters.} A $\kappa\times\lambda$ scan over fifteen cells
on the matched replay is positive across a wide plateau and collapses only
where a small saturation constant meets a large strength
($\kappa{=}0.5,\lambda{=}0.9$: $-19.95$). At $\kappa{=}10$ and $\kappa{=}30$
every strength is positive, and the selected cell ($\kappa{=}3$,
$\lambda{=}0.5$, $+1.48$) sits inside the plateau; the scan was not used for
selection. The plateau motivates evaluating both the update rule and its operating
strength within the same host score space.

\section{Boundary Analyses}
\label{app:boundary}

\begin{table}[t]
\caption{\textbf{Deployment diagnostics for persistent index adaptation.}
Rows combine native-host results with explicitly marked BM25 replay
interventions on granularity and store content.
Dashes denote unmeasured or undefined quantities.}
\label{tab:conditions}
\centering\small
\setlength{\tabcolsep}{4pt}
\resizebox{\linewidth}{!}{%
\begin{tabular}{llrrlrl}
\toprule
\textbf{Deployment} & \textbf{Host} & \textbf{q / store} & \textbf{RO} & \textbf{Separable unit} & \textbf{Headroom} & \textbf{Outcome} \\
\midrule
LoCoMo & BM25 & 154 & $0.74$ & yes (turns) & $27.5\%$ & helps \\
LoCoMo & Mem0 & 154 & -- & yes (facts) & -- & helps \\
LoCoMo & MemoryOS & 154 & $0.81$ & pages ($\leq14$ turns) & $22.8\%$ & helps \\
LoCoMo & A-MEM & 154 & -- & roots expand into links & -- & all policies improve F1 \\
\midrule
MemBench, history adapter & BM25 & \textbf{1 final} & -- & yes & -- & count loses; others gain \\
GroupMem (4 domain stores) & BM25 & 145--214 & -- & yes & -- & count loses; others gain \\
\midrule
LoCoMo, $21\times$ self-duplicated & BM25 & 154 & -- & \textbf{no} (copies of evidence) & -- & count loses in this replay \\
\bottomrule
\end{tabular}}
\end{table}

\subsection{Why volume is not the variable}
\label{app:volume}

The dilution model predicts that a larger store is a less accessible one. The BM25 replay
shows that size alone is insufficient, because added text is not independent of the
query: under BM25 far-domain documents raise the document frequency of their
own vocabulary and leave the query's terms rare, so the inverse document
frequency of the terms that matter rises. Dilution and sharpening act in
opposite directions, and which dominates depends on the relation between the
added text and the query rather than on its quantity
(\LinkFigurePanel{fig:evolvability-main}{a}).

A $36$-cell BM25 grid on the matched replay (three distractor families
$\times$ four volumes $\times$ three seeds, every
cell $n{=}1{,}527$ with clustered intervals) makes the separation
quantitative for the rule itself. Under BM25, near-domain decay deepens
monotonically with volume ($-3.1$, $-4.9$, $-6.7$, $-7.7$ points at
$1\times$, $3\times$, $10\times$, $21\times$) while the count gain stays flat
at $+1.3$ to $+1.7$, and far-domain decay never exceeds $2.0$ points at any
volume: the rule buys back a roughly constant share of a cost that grows with
confusability, not with size. Two regimes locate its limits. Duplicating each
conversation's own content grows the decay from $-2.0$ to $-6.3$ points and
turns the count gain from $+1.2$ at $1\times$ into $-1.2$ at $10\times$ and
beyond, because the persistent competitors are then copies of the evidence
itself. These controlled additions separate store size from the competition induced
by the added content (\LinkFigurePanel{fig:evolvability-main}{a}).

\subsection{Structural limits and empirical deployment diagnostics}

\LinkTable{tab:conditions} separates structural limits from empirical
properties. Future retrieval and pool membership constrain what an update
can affect; occupancy, granularity, and calibration are measured risk factors.
\emph{Return to the store} is the number of questions evaluated against one
persistent state: ten conversation stores and $1{,}540$ questions on LoCoMo,
one terminal question per MemBench history, preceded by the adapter's history
retrieval events (\LinkAppendix{app:datasets}). With zero initial state and
no such prior events, a post-sealing update cannot change the only answer. \emph{Contested slots} refers to repeated occupancy of the native feedback
list. \emph{Separable units} concerns the coupling between a selected unit
and the useful evidence it carries or expands into. \emph{Reachable headroom}
is the share of gold outside the returned context but inside the candidate
pool. Together these quantities describe the room for access adaptation;
headroom additionally requires gold labels.

\section{Policy Inference and State Cost}
\label{app:reranker}

\begin{table}[htbp]
\caption{\textbf{Recorded policy cost on native Mem0}, $1{,}540$ queries per
arm. ``Counters'' is the number of memories the state ever touches, and
``state bytes'' the serialised size of the persistent index state.}
\label{tab:costs}
\centering\small
\begin{tabular}{llrrrrr}
\toprule
\textbf{Policy} & \textbf{Seed} & \textbf{Model evaluations} & \textbf{Median (ms)} & \textbf{P95 (ms)} & \textbf{Counters} & \textbf{State bytes} \\
\midrule
off & -- & 0 & - & - & 0 & 2 \\
count & -- & 0 & - & - & 235 & 9{,}669 \\
frozen & -- & 1{,}540 & 23{,}997 & 25{,}400 & 152 & 6{,}255 \\
\multirow{3}{*}{trained} & 20260908 & 1{,}540 & 87.3 & 88.9 & 190 & 7{,}818 \\
 & 20260909 & 1{,}540 & 86.9 & 88.2 & 187 & 7{,}712 \\
 & 20260910 & 1{,}540 & 86.6 & 88.3 & 216 & 8{,}890 \\
\bottomrule
\end{tabular}

\end{table}

Frozen-policy times include API pacing, retries, and network time. Student
times cover the local \texttt{probabilities()} call, including device-to-CPU
transfer; historical GPU identity was not recorded. Both exclude fitting,
model loading, retrieval, and answer generation. The answerer used Qwen3-32B,
BF16, deterministic decoding, batch size one, and at most $128$ output tokens.
State bytes measure compact UTF-8 JSON with memory-ID keys after updates,
reset per conversation; they are not resident RAM. The serving conditions
therefore support a recorded latency comparison, not a hardware-normalised
speedup or an end-to-end agent cost estimate.

\Needspace{12\baselineskip}
\section{Proof of Relative Access Monotonicity}
\label{app:relative-access-proof}

\begin{equation}
\Delta s_t(q,i) \;=\; -\,\lambda_p\,\hat\sigma(q,i)\,\Delta u_{i,t},
\qquad
\Delta u_{i,t}\ge 0\ \Longrightarrow\ \Delta s_t(q,i)\le 0 .
\label{eq:monotone}
\end{equation}
\begin{proposition}[Relative access monotonicity]
\label{prop:margin}
Fix a query $q$, a finite candidate pool, its nonnegative normalised base
scores, a strength $\lambda_p\in[0,1]$, and a total order on candidate IDs
for breaking ties. Let $m^\star$ be a target in that pool, and let one update satisfy
$\Delta u_{\star,t}=0$ and $\Delta u_{j,t}\ge 0$ for all $j\ne\star$. Then
under \LinkEquation{eq:score} no competitor's score increases, so the rank of
$m^\star$ cannot worsen; and, if $\lambda_p>0$, for every competitor $m_j$ with
$\hat\sigma(q,j)>0$ and $\Delta u_{j,t}>0$, the pairwise margin
$\Delta_{\star j,t}(q)=s_t(q,m^\star)-s_t(q,m_j)$ strictly increases.
\end{proposition}

We prove \LinkProposition{prop:margin} for the same query and candidate
pool before and after one state update. Scores are sorted in descending
order; numerical rank starts at one, so a smaller rank is better.
The base scores, suppression strength, and ID order used to break ties
are held fixed throughout this comparison.

\begin{proof}[Proof of \LinkProposition{prop:margin}]
Write $b_i=\hat\sigma(q,i)\ge0$ and
$\delta_i=u_{i,t+1}-u_{i,t}$. By assumption, $\delta_\star=0$ and
$\delta_j\ge0$ for every competitor $j$. From \LinkEquation{eq:score},
\begin{equation}
s_{t+1}(q,i)-s_t(q,i)=-\lambda_p b_i\delta_i.
\label{eq:relative-score-change}
\end{equation}
Thus the target's score is unchanged and every competitor's score is
nonincreasing. Subtracting the competitor's score from the target's gives
\begin{equation}
\Delta_{\star j,t+1}(q)-\Delta_{\star j,t}(q)
=\lambda_p b_j\delta_j\ge0.
\label{eq:relative-margin-change}
\end{equation}
This inequality is strict whenever $\lambda_p>0$, $b_j>0$, and
$\delta_j>0$, proving the pairwise-margin statement.

To establish the ranking statement, let $P_t$ be the set of candidates
that precede $m^\star$ at step $t$: their scores are higher, or equal
with their IDs preceding the target under the fixed tie order. Take any
$j\notin P_t$, $j\ne\star$. If its score is strictly lower than the
target's, it remains lower because it cannot increase and the target's
score stays fixed. If the scores are equal, $j$ already loses the tie;
after the update it either has a lower score or still loses the same tie.
Hence no such candidate can enter $P_{t+1}$, and
\begin{equation}
P_{t+1}\subseteq P_t,
\qquad r_{t+1}(m^\star)=1+|P_{t+1}|\le1+|P_t|=r_t(m^\star).
\label{eq:relative-rank-change}
\end{equation}
The target's numerical rank therefore cannot increase, completing the proof.
\end{proof}

\paragraph{Top-$k$ consequence and scope.}
For a fixed budget $k$, a target already in the top $k$ remains there.
A previously excluded target enters exactly when
$|P_{t+1}|<k\le|P_t|$. Strict pairwise-margin improvement alone does not
ensure this crossing: competitors may still remain ahead. If
$\lambda_p=0$, all scores and ranks are unchanged. The result compares two
states on the same query; it does not assert monotonic ranks across
different queries. It also does not cover a changing candidate pool,
base scores, strength, tie order, or an update that suppresses the target
itself. Even under the stated conditions, preserving one target's access
does not guarantee higher aggregate recall or downstream utility, since
displaced memories may also be useful.

\Needspace{.60\textheight}
\section{Prior Systems by What Adapts}
\label{app:related-table}

\LinkTable{tab:related} distinguishes native adaptation from an independently
controlled index intervention in the systems we evaluate as hosts. \emph{Persistent access adaptation} includes
changes to links, access metadata, or retrieval controls that affect later
queries; it is not exclusive to AMIE. \emph{Separate index-only policy} means
that learned actions explicitly update persistent access state through
the host's retrieval interface. ``NR'' means that this separate
control is not reported in the cited design, not that the system has no index
adaptation. Our contribution is the isolation and plug-in evaluation of this
axis, not its first occurrence.

\begin{table}[!htbp]
\caption{Native memory and index adaptation mechanisms in the evaluated hosts.
NR: separate index-only policy not reported.}
\label{tab:related}
\centering\footnotesize
\setlength{\tabcolsep}{3pt}
\renewcommand{\arraystretch}{1.12}
\begin{tabular}{@{}>{\raggedright\arraybackslash}p{.22\linewidth}>{\raggedright\arraybackslash}p{.24\linewidth}>{\raggedright\arraybackslash}p{.28\linewidth}>{\raggedright\arraybackslash}p{.18\linewidth}@{}}
\toprule
\textbf{System} & \textbf{Native adaptation target} & \textbf{Persistent access adaptation} & \textbf{Separate index-only policy} \\
\midrule
BM25 & Static sparse retrieval & None & NR \\
Mem0~\citep{mem0} & Fact consolidation & Updated memory records & NR \\
MemoryOS~\citep{kang2025memoryos} & Tiered memory management & Heat and placement & NR \\
A-MEM~\citep{xu2025amem} & Notes and linked metadata & Links coupled to note evolution & NR \\
Memory-R2~\citep{yan2026memoryr2faircreditassignment} & Learned memory management & Reasoned insert, update, and delete operations & NR \\
\midrule
\textbf{AMIE} & \textbf{Persistent access weights} & \textbf{Count and semantic suppression} & \textbf{Explicitly isolated} \\
\bottomrule
\end{tabular}

\end{table}

\end{document}